%% file: ijcai23.tex
 \pdfoutput=1

\documentclass{article}
\usepackage{ijcai23}

\usepackage{times}
\usepackage{epsfig}
\usepackage{soul}
\usepackage{scalerel}
\usepackage{amsmath,amsfonts,amssymb,amsthm,url}
\usepackage[hidelinks]{hyperref}
\usepackage[utf8]{inputenc}
\usepackage[small]{caption}
\usepackage{nicefrac}
\usepackage{mathtools}
\usepackage{booktabs}
\usepackage{array,multirow}
\usepackage{graphicx}
\usepackage[ruled,noend]{algorithm2e}
\usepackage[switch]{lineno}
\usepackage{caption}
\usepackage{subcaption}
\usepackage{enumitem}
\usepackage{dsfont}
\usepackage{cancel}
\usepackage{csquotes}
\usepackage{listings}
\usepackage{setspace}

\graphicspath{{./figures/}}

\usepackage{latexsym}
\DeclareFontFamily{OT1}{pzc}{}
\DeclareFontShape{OT1}{pzc}{m}{it}{<-> s * [1] pzcmi7t}{}
\DeclareMathAlphabet{\mathpzc}{OT1}{pzc}{m}{it}
\DeclarePairedDelimiter\abs{\lvert}{\rvert}%
\usepackage{xcolor}
\definecolor{codegreen}{rgb}{0,0.6,0}
\definecolor{codegray}{rgb}{0.5,0.5,0.5}
\definecolor{codepurple}{rgb}{0.58,0,0.82}
\definecolor{backcolour}{rgb}{0.95,0.95,0.92}
\lstdefinestyle{mystyle}{
    backgroundcolor=\color{backcolour},   
    numberstyle=\tiny\color{codegray},
    basicstyle=\footnotesize,
    breakatwhitespace=false,         
    breaklines=true,                 
    captionpos=b,                    
    keepspaces=true,                 
    numbers=left,                    
    numbersep=6pt,                  
    showspaces=false,                
    showstringspaces=false,
    showtabs=false,                  
    tabsize=2
}
\newtheorem{example}{Example}
\newtheorem{theorem}{Theorem}

\newtheorem{lemma}{Lemma}
\newtheorem{definition}{Definition}

\title{Sample Efficient Model-free Reinforcement Learning from \\LTL Specifications with Optimality Guarantees}

\author{
Daqian Shao
\And
Marta Kwiatkowska\\
\affiliations
 Department of Computer Science, University of Oxford, UK\\
\emails
\{daqian.shao, marta.kwiatkowska\}@cs.ox.ac.uk
}

\begin{document}

\maketitle

\begin{abstract}
Linear Temporal Logic (LTL) is widely used to specify high-level objectives for system policies, and it is highly desirable for autonomous systems to learn the optimal policy with respect to such specifications. However, learning the optimal policy from LTL specifications is not trivial. We present a model-free Reinforcement Learning (RL) approach that efficiently learns an optimal policy for an unknown stochastic system, modelled using Markov Decision Processes (MDPs). We propose a novel and more general product MDP, reward structure and discounting mechanism that, when applied in conjunction with off-the-shelf model-free RL algorithms, efficiently learn the optimal policy that maximizes the probability of satisfying a given LTL specification with optimality guarantees. We also provide improved theoretical results on choosing the key parameters in RL to ensure optimality. To directly evaluate the learned policy, we adopt probabilistic model checker PRISM to compute the probability of the policy satisfying such specifications. Several experiments on various tabular MDP environments across different LTL tasks demonstrate the improved sample efficiency and optimal policy convergence.
\end{abstract}

\input{macro}

\section{Introduction}

Linear Temporal Logic (LTL) is a temporal logic language that can encode formulae regarding properties of an infinite sequence of logic propositions. LTL is widely used for the formal specification of high-level objectives for robotics and multi-agent systems, and it is desirable for a system or an agent in the system to learn policies with respect to these high-level specifications. Such systems are modelled as Markov Decision Processes (MDPs), where classic policy synthesis techniques can be adopted if the states and transitions of the MDP are known. However, when the transitions are not known \textit{a priori}, the optimal policy needs to be learned through interactions with the MDP.

Model-free Reinforcement Learning~\cite{Sutton2018}, a powerful method to train an agent to choose actions in order to maximize rewards over time in an unknown environment, is a perfect candidate for LTL specification policy learning. However, it is not straightforward to utilize reward-based RL to learn the optimal policy that maximizes the probability of satisfying LTL specifications~\cite{Alur2021ALearning} due to the difficulty of deciding when, where, and how much reward to give to the agent. To this end, most works adopt a method to first transform the LTL specification into an automaton, then build a product MDP using the original environment MDP and the automaton, on which model-free RL algorithms are applied. However, one crucial obstacle still remains, and that is how to properly define the reward that leads an agent to the optimal satisfaction of LTL specification. Several algorithms have been proposed~\cite{Hahn2019Omega-regularLearning,Bozkurt2019ControlLearning,Hasanbeig2019ReinforcementGuarantees} for learning LTL specifications, where, in order to ensure optimality, the key parameters are chosen depending on assumptions or knowledge of the environment MDP. These works normally only prove the existence of the optimality guarantee parameters or provide unnecessarily harsh bounds for them, which might lead to inefficient learning. In addition, it is unclear how to explicitly choose these parameters even with certain knowledge of the environment MDP. Furthermore, the assumed parameters are evaluated in experiments indirectly, either through inspection of the value function or comparison of the expected reward gained, making it difficult to tune the optimality parameters for LTL learning in general MDPs.

In this work, we propose a novel and more general product MDP, reward structure and discounting mechanism that, by leveraging model-free reinforcement learning algorithms, efficiently learns the optimal policy that maximizes the probability of satisfying the LTL specification with guarantees. We demonstrate improved theoretical results on the optimality of our product MDP and the reward structure, with more stringent analysis that yields better bounds on the optimality parameters. Moreover, this analysis sheds light on how to explicitly choose the optimality parameters based on the environment MDP. We also adopt counterfactual imagining that exploits the known high-level LTL specification to further improve the performance of our algorithm. Last but not least, we propose to use the PRISM model checker~\cite{Kwiatkowska2011PRISMSystems} to directly evaluate the satisfaction probability of the learned policies, providing a platform to directly compare algorithms and tune key parameters. We conduct experiments on several common MDP environments with various challenging LTL tasks, and demonstrate the improved sample efficiency and convergence of our methods.

Our contributions include: (i) a novel product MDP design that incorporates an accepting states counter with a generalized reward structure; (ii) a novel reinforcement learning algorithm that converges to the optimal policy for satisfying LTL specifications, with theoretical optimality guarantees and theoretical analysis results on choosing the key parameters; (iii) the use of counterfactual imagining, a method to exploit the known structure of the LTL specification by creating imagination experiences through counterfactual reasoning; and (iv) direct evaluation of the proposed algorithms through a novel integration of probabilistic model checkers within the evaluation pipeline, with strong empirical results demonstrating better sample efficiency and training convergence.

\subsection*{Related Work}

Most works on LTL learning with reward-based RL utilize a product MDP: a product of the environment MDP and an automaton translated from the LTL specification. Sadigh \textit{et al.}~\shortcite{Sadigh2014ASpecifications} first used deterministic Rabin automata to create this product with a discounted reward design to learn LTL, while later works adopted a new automaton design, limit-deterministic Büchi automata (LDBA)~\cite{Sickert2016Limit-deterministicLogic}. Hahn \textit{et al.}~\shortcite{Hahn2019Omega-regularLearning} adopted a product MDP with LDBA and augmented it with sink states to reduce the LTL satisfaction problem into a limit-average reward problem with optimality guarantees. Hahn \textit{et al.}~\shortcite{Hahn2020FaithfulObjectives} later modified this approach by including two discount factors with similar optimality guarantee results. Bozkurt \textit{et al.}~\shortcite{Bozkurt2019ControlLearning} proposed a discounted reward learning algorithm on the product MDP with optimality guarantees, where the discount factor is chosen based on certain assumptions about the unknown environment MDP. To the best of our knowledge, these approaches are the only available to provide optimality guarantees for the full infinite-horizon LTL learning. However, many methods have nevertheless demonstrated empirical results for learning LTL. Hasanbeig \textit{et al.}~\shortcite{Hasanbeig2020DeepLogics,Hasanbeig2019ReinforcementGuarantees} proposed an accepting frontier function as the reward for the product MDP, while Cai \textit{et al.}~\shortcite{Cai2021ModularLogic} extended this reward frontier to continuous control tasks.

Due to the difficulty of learning full LTL, many approaches focus on learning restricted finite LTL variants. Giacomo \textit{et al.}~\shortcite{DeGiacomo2013LinearTraces,GiuseppeDeGiacomo2019FoundationsScheduling} proposed the LTLf variant and a corresponding reinforcement learning algorithm; Littman \textit{et al.}~\shortcite{Littman2017Environment-IndependentGLTL} formulated a learning algorithm for the GLTL variant; Aksaray \textit{et al.}~\shortcite{Aksaray2016Q-LearningSpecifications} proposed to learn Signal Temporal Logic and Li \textit{et al.}~\shortcite{Li2016ReinforcementRewards} a truncated LTL variant for robotics applications.

Another related line of work leverages automata to learn non-Markovian rewards. Toro Icarte \textit{et al.}~\shortcite{Icarte2022RewardLearning,ToroIcarte2018UsingLearning} defined a reward machine automaton to represent high-level non-Markovian rewards, while Camacho \textit{et al.}~\shortcite{Camacho2019LTLLearning} introduced a method to learn finite LTL specifications by transforming them into reward machines. However, the expressiveness of reward machines is strictly weaker than that of LTL. Lastly, there are works that exploit other high-level logic specifications to facilitate learning~\cite{Andreas2016ModularSketches,Jiang2021Temporal-Logic-BasedTasks,Jothimurugan2020ATasks,Jothimurugan2021CompositionalSpecifications}, but they are less relevant to reinforcement learning from LTL.


\section{Preliminaries}

Before formulating our problem, we provide preliminary background on Markov decision processes, linear temporal logic, and reinforcement learning.

\subsection{Markov Decision Processes}

\begin{definition}[Markov decision process~\cite{Littman2001MarkovProcesses}]
A Markov decision process (MDP) $\MDP$ is a tuple $(\states,s_0,\actions,\transitions,\propositions,\labFunc,\reward,\discount)$, where $\states$ is a finite set of states, $s_0 \in \states$ is the initial state, $\actions$ is a finite set of actions, $\transitions:\states\times\actions\times\states\rightarrow [0,1]$ is the probabilistic transition function, $\propositions$ is the set of atomic propositions, $\labFunc: \states\rightarrow2^{\propositions}$ is the proposition labeling function, $\reward:\states\times\actions\times\states\rightarrow\realNumber$ is a reward function and $\discount:\states\rightarrow (0,1]$ is a discount function. Let $\actions(\state)$ denote the set of available actions at state $\state$, then, for all $\state\in\states$, it holds that $\sum_{\state^\prime\in\states}\transitions(\state,\action,\state^\prime)=1$ if $\action\in\actions(\state)$ and 0 otherwise.
\end{definition}

An infinite path is a sequence of states $\MDPpath=\state_0,\state_1,\state_2...$, where there exist $\action_{i+1}\in\actions(\state_{i})$ such that $\transitions(\state_i,\action_{i+1},\state_{i+1})>0$ for all $i\geq 0$, and a finite path is a finite such sequence. We denote the set of infinite and finite paths of the MDP $\MDP$ as $\paths$ and $\Fpaths$, respectively. We use $\MDPpath[i]$ to denote $\state_i$, and  $\MDPpath[:i]$ and $\MDPpath[i:]$ to denote the prefix and suffix of the path, respectively. Furthermore, we assume self-loops: if $A(\state)=\varnothing$ for some state $\state$, we let $\transitions(\state,\action,\state)=1$ for some $\action\in\actions$ and $A(\state)={\action}$ such that all finite paths can be extended to an infinite one.

{\setstretch{1.1}A finite-memory policy $\policy$ for $\MDP$ is a function $\policy:\Fpaths\rightarrow D(\actions)$ such that $supp(\policy(\MDPpath))\subseteq \actions(\MDPpath[-1])$,
where $D(\actions)$ denotes a distribution over $\actions$, $supp(d)$ denotes the support of the distribution and $\MDPpath[-1]$ is the last state of a finite path $\MDPpath$. A policy $\policy$ is memoryless if it only depends on the current state, \textit{i.e.,} $\MDPpath[-1]=\MDPpath^\prime[-1]$ implies $\policy(\MDPpath)=\policy(\MDPpath^\prime)$, and a policy is deterministic if $\policy(\MDPpath)$ is a point distribution for all $\MDPpath\in\Fpaths$. For a deterministic memoryless policy, we let $\policy(\state)$ represent $\policy(\MDPpath)$ where $\MDPpath[-1]=\state$.

Let $\paths_\policy\subseteq\paths$ denote the subset of infinite paths that follow policy $\policy$ and we define the probability space $(\paths_\policy,\filtration_{\paths_\policy},\probP_\policy)$ over $\paths_\policy$ in the standard way. Then, for any function $f:\paths_\policy\rightarrow\realNumber$, let $\expectE_\policy[f]$ be the expectation of $f$ over the infinite paths of $\MDP$ following $\policy$.

A Markov chain (MC) induced by $\MDP$ and deterministic memoryless policy $\policy$ is a tuple $\MDP_\policy=(\states,s_0,\transitions_\policy,\propositions,\labFunc)$, where $\transitions_\policy(\state,\state^\prime)=\transitions(\state,\policy(\state),\state^\prime)$. A sink (bottom) strongly connected component (BSCC) of a MC is a set of states $C \subseteq \states$ such that, for all pairs $\state_1,\state_2 \in C$, there exists a path from $\state_1$ to $\state_2$ following the transition function $\transitions_\policy$ (strongly connected), and there exists no state $\state^\prime\in S\setminus C$ such that $\transitions_\policy(\state,\state^\prime)>0$ for all $\state\in C$ (sink).}

\subsection{Linear Temporal Logic}
Linear Temporal Logic (LTL) provides a high-level description for specifications of a system. LTL is very expressive and can describe specifications with infinite horizon.
\begin{definition}[\cite{Baier2008}]
An LTL formula over atomic propositions $\propositions$ is defined by the grammar:
\begin{equation*}
    \LTL \vcentcolon\vcentcolon= true\mid p\mid\LTL_1\land\LTL_2\mid\neg\LTL\mid \textsf{\upshape X }\LTL\mid\LTL_1 \textsf{\upshape U } \LTL_2, \quad p\in\propositions,
\end{equation*}
where $\textsf{\upshape X}$ represents next and $\textsf{\upshape U}$ represents until. Other Boolean and temporal operators are derived as follows: or: $\LTL_1\lor\LTL_2=\neg(\neg\LTL_1\land\neg\LTL_2)$; implies: $\LTL_1\rightarrow\LTL_2=\neg\LTL_1\lor\LTL_2$; eventually: $\textsf{\upshape F }\LTL=true \textsf{\upshape U }\LTL$; and always: $\textsf{\upshape G }\LTL=\neg(\textsf{\upshape F } \neg\LTL)$.
\end{definition}

The satisfaction of an LTL formula $\LTL$ by an infinite path $\MDPpath\in\paths$ is denoted by $\MDPpath\models\LTL$, and is defined by induction on the structure of $\LTL$:
\begin{align*}
    \MDPpath\text{ satisfies $\LTL$ if }\quad p\in\labFunc(\MDPpath[0])\quad &\text{ for } p\in\propositions;\\ 
    \MDPpath[1:]\models\LTL\quad &\text{ for } \textsf{\upshape X }\LTL;\\
    \exists i, \MDPpath[i]\models\LTL_2 \text{ and } \forall j<i, \MDPpath[j]\models\LTL_1\quad &\text{ for } \LTL_1\textsf{\upshape U }\LTL_2 ,
\end{align*}
with the satisfaction of Boolean operators defined by their default meaning.

\subsection{Reinforcement Learning}

Reinforcement learning~\cite{Sutton2018} teaches an agent in an unknown environment to select an action from its action space, in order to maximize rewards over time. 
In most cases the environment is modelled as an MDP $\MDP=(\states,s_0,\actions,\transitions,\propositions,\labFunc,\reward,\gamma)$. Given a deterministic memoryless policy $\policy$, at each time step $t$, let the agent's current state be $\state_t$, then the action $\action=\policy(\state_t)$ is chosen and the next state $\state_{t+1}\thicksim\transitions(\state_t,\action,\cdot)$ together with the immediate reward $\reward(\state_t,\action,\state_{t+1})$ is received from the environment. Then, starting at $\state\in\states$ and time step $t$, the expected discounted reward following $\policy$ is
\begin{equation}
\ExpReward^\policy_t(\state)=\expectE_\policy[\sum^\infty_{i=t}(\prod_{j=t}^{i-1}\discount(\state_{j}))\cdot\reward(\state_{i},\action_{i},\state_{i+1})\mid \state_t=\state],
\end{equation}
where $\prod_{j=t}^{t-1}\defeq1$. The agent's goal is to learn the optimal policy that maximizes the expected discounted reward. Note that we defined a discount function instead of a constant discount factor because it is essential for our proposed LTL learning algorithm to discount the reward depending on the current MDP state.

Q-learning~\cite{Watkins1992Q-learning} is a widely used approach for model-free RL. It utilizes the idea of the Q function $Q^\policy(\state,\action)$, which is the expected discounted reward of taking action $\action$ at state $\state$ and following policy $\policy$ after that. The Q function for all optimal policies satisfies the Bellman optimality equations:
\begin{align}
Q^*(\state,\action)&=\sum_{\state^\prime\in\states}{\transitions(\state,\action,\state^\prime)\big(\reward(\state,\action,\state^\prime)}\\
&{+\discount(\state) \max_{\action^\prime\in\actions}{Q^*(\state^\prime,\action^\prime)}\big)}
\;\forall\action\in\actions, \state\in\states. \nonumber
\end{align}
At each iteration of the Q-learning algorithm, the agent's experiences, i.e, the next state $\state^\prime$ and immediate reward $\reward(\state,\action,\state^\prime)$, are used to update the Q function:
\begin{equation}
Q(\state,\action)\overset{\alpha}{\gets}\reward(\state,\action,\state^\prime)+\discount(\state)\max_{\action^\prime\in\actions}Q(\state^\prime,\action^\prime),
\end{equation}
where $\alpha$ is the learning rate and $x\overset{\alpha}{\gets}y$ represents $x\gets x+\alpha(y-x)$. In addition, the optimal policy can be recovered from the optimal Q function $Q^*(\state,\action)$ by selecting the action $\action$ with the highest state-action pair value in each state $\state$. Q-learning converges to the optimal Q function in the limit given that each state-action pair is visited infinitely often~\cite{Watkins1992Q-learning}, and thus learns the optimal policy.

\section{Our Method}

Our goal is to formulate a model-free reinforcement learning approach to efficiently learn the optimal policy that maximizes the probability of satisfying an LTL specification with guarantees. We now give an overview of our method. We first transform the LTL objective $\LTL$ into a limit-deterministic Büchi automaton. Then, we introduce a novel product MDP and define a generalized reward structure on it. With this reward structure, we propose a Q-learning algorithm that adopts a collapsed Q function to learn the optimal policy with optimality guarantees. Lastly, we enhance our algorithm with counterfactual imagining that exploits the automaton structure to improve performance while maintaining optimality.

\subsection{Problem Formulation}
Given an MDP $\MDP=(\states,s_0,\actions,\transitions,\propositions,\labFunc,\reward,\gamma)$ with unknown states and transitions and an LTL objective $\LTL$, for any policy $\policy$ of the MDP $\MDP$, let $\probP_\policy(\state\models\LTL)$ denote the probability of paths from state s following $\policy$ satisfying the LTL formula $\LTL$:
\begin{equation}
\probP_\policy(\state\models\LTL)=\probP_\policy\{\MDPpath\in\paths_\policy\mid\MDPpath[0]=s,\MDPpath\models\LTL\}.
\end{equation}
Then, we would like to design a model-free RL algorithm that learns a deterministic memoryless optimal policy $\policy_\LTL$ that maximizes the probability of $\MDP$ satisfying $\LTL$:
\begin{equation}
    \probP_{\policy_\LTL}(\state\models\LTL)=\max_\policy\probP_\policy(\state\models\LTL) \quad \forall\state\in\states.
\end{equation}

\subsection{Limit-deterministic Büchi automata}

We first transform the LTL specifications into automata. The common choices of automata include deterministic Rabin automata and non-deterministic Büchi automata. In this work, we adopt a Büchi automata variant called limit-deterministic Büchi automata (LDBA)~\cite{Sickert2016Limit-deterministicLogic}.

\begin{definition}
A non-deterministic Büchi automaton is an automaton $\automaton=(\propositions,\autoStates,\autoState_0,\autoTransitions,\autoAccept)$, where $\propositions$ is the set of atomic propositions, $\autoStates$ is a finite set of states, $\autoState_0\in\autoStates$ is the initial state and $\autoAccept\subseteq\autoStates$ is the set of accepting states. Let $\Sigma=2^\propositions\cup\{\epsilon\}$ be a finite alphabet, then the transition function is given by $\autoTransitions:\autoStates\times\Sigma\rightarrow 2^{\autoStates}$.
\end{definition}

\begin{definition}[LDBA]
A Büchi automaton is limit-deterministic if $\autoStates$ can be partitioned into a deterministic set and a non-deterministic set, that is, $\autoStates = \autoStates_\mathcal{N} \cup \autoStates_\mathcal{D}$, where $\autoStates_\mathcal{N} \cap \autoStates_\mathcal{D}=\varnothing$, such that

\begin{enumerate}
    \item $\autoAccept\subseteq\autoStates_\mathcal{D}$ and $\autoState_0\in\autoStates_\mathcal{N}$;
    \item $\abs{\autoTransitions(\autoState,\alphabet)}\leq1$ for all $\autoState\in\autoStates_\mathcal{N}$ and $\alphabet\neq\epsilon\in\alphabets$;
    \item $\autoTransitions(\autoState,\alphabet)\subseteq\autoStates_\mathcal{D}$ and $\abs{\autoTransitions(\autoState,\alphabet)}\leq1$ for all $\autoState\in\autoStates_\mathcal{D}$ and $\alphabet\in\alphabets$;
\end{enumerate}
\end{definition}

An LDBA is a Büchi automaton where the non-determinism is limited in the initial component $\autoStates_\mathcal{N}$ of the automaton. An LDBA starts in a non-deterministic initial component and then transitions into a deterministic accepting component $\autoStates_\mathcal{D}$ through $\epsilon$-transitions after reaching an accepting state, where all transitions after this point are deterministic. We follow the formulation of Bozkurt \textit{et al.}~\shortcite{Bozkurt2019ControlLearning} to extend the alphabets with an $\epsilon$-transition that handles all the non-determinism, meaning only $\epsilon$ can transition the automaton state to more then 1 states: $\abs{\autoTransitions(\autoState,\epsilon)}>1$. This allows the MDP to synchronise with the automaton, which we will discuss in detail in Section~\ref{sec:product_MDP}.

An infinite word $\word\in\Sigma^\omega$, where $\Sigma^\omega$ is the set of all infinite words over the alphabet $\Sigma$, is accepted by a Büchi automaton $\automaton$ if there exists an infinite automaton run $\Autorun\in\autoStates^\omega$ from $\autoState_0$, where $\Autorun[t+1]\in\autoTransitions(\Autorun[t],\word[t]), \forall t\geq 0$, such that $\text{inf}(\Autorun)\cap\autoAccept\neq\varnothing$, where $\text{inf}(\Autorun)$ is the set of automaton states that are visited infinitely often in the run $\Autorun$.

LDBAs are as expressive as the LTL language, and the satisfaction of any given LTL specification $\LTL$ can be evaluated on the LDBA derived from $\LTL$. We use Rabinizer 4~\cite{Kretinsky2018RabinizerAutomaton} to transform LTL formulae into LDBAs. In Figure~\ref{fig:LDBA} we give an example of the LDBA derived from the LTL formula \enquote{\textsf{\upshape F}\textsf{\upshape G }a \& \textsf{\upshape G }!c}, where state 1 is the accepting state. LDBAs are different from reward machines~\cite{Icarte2022RewardLearning} because they can express properties satisfiable by infinite paths, which is strictly more expressive than reward machines, and they have different accepting conditions.

\begin{figure*}[tb]
\begin{subfigure}[t]{0.33\textwidth}
    \centering
        \includegraphics[width=0.5\textwidth]{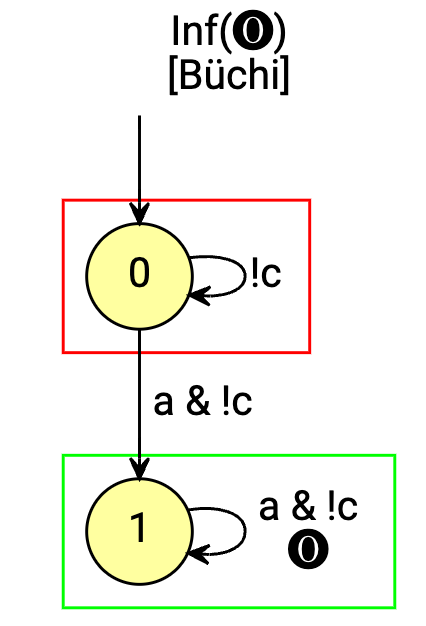}
    \caption{LDBA for \enquote{\textsf{\upshape F}\textsf{\upshape G} a \& \textsf{\upshape G} !c}.}
    \label{fig:LDBA}
\end{subfigure}
\begin{subfigure}[t]{0.62\textwidth}
    \centering
    \includegraphics[width=0.7\textwidth]{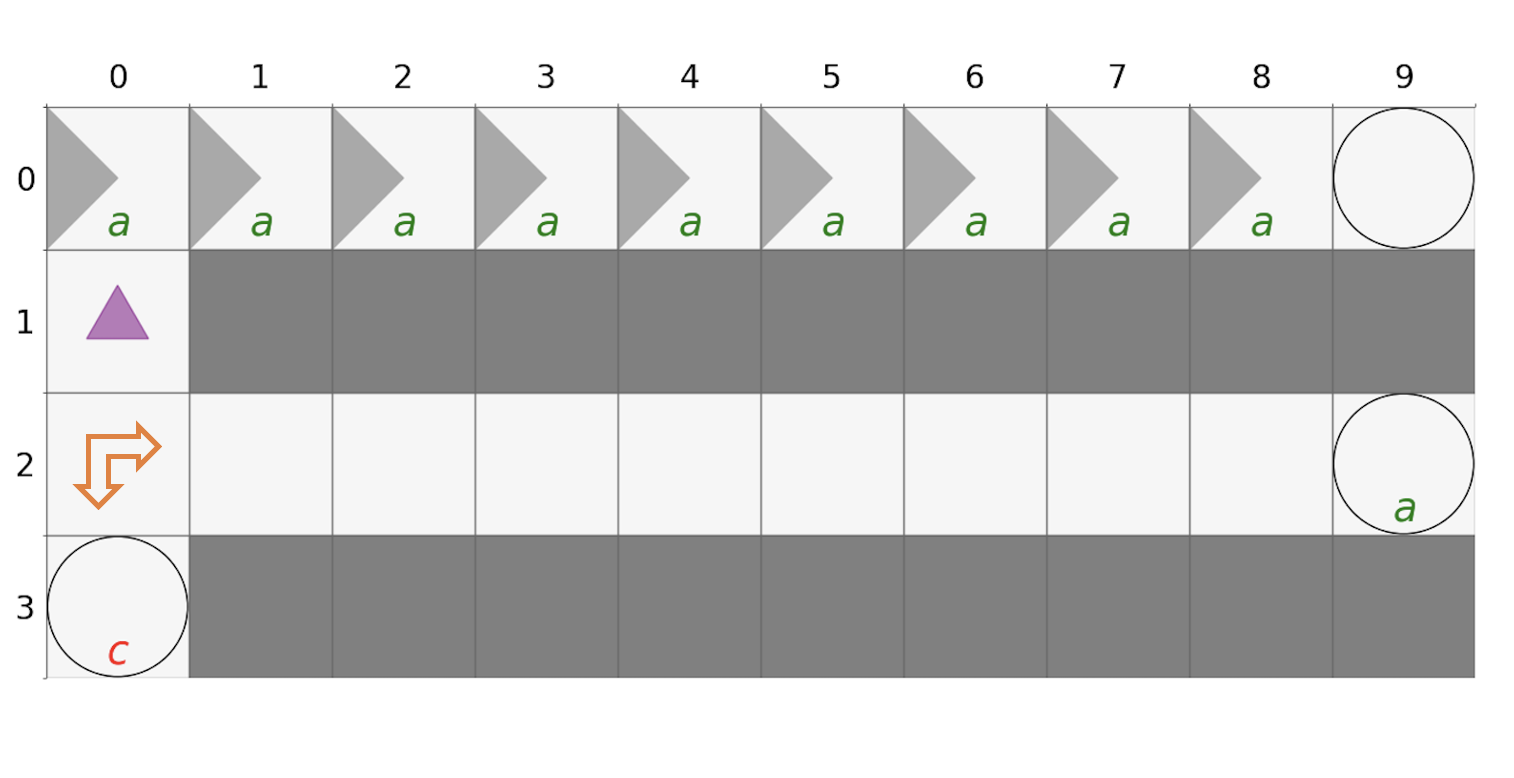}
    \caption{A motivating MDP example for the $K$ counter, where the purple triangle at (1,0) is the starting point and the bidirectional arrow at (2,0) is a probability gate.}
    \label{fig:example_mdp}
\end{subfigure}
\caption{An example LDBA (left) and probabilistic gate MDP (right) motivating the $K$ counter (see Example~\ref{example_mdp}).}
\label{fig:example_task}
\end{figure*}

\subsection{Product MDP}
\label{sec:product_MDP}
In this section, we propose a novel product MDP of the environment MDP, an LDBA and an integer counter, where the transitions for each component are synchronised. Contrary to the standard product MDP used in the literature~\cite{Bozkurt2019ControlLearning,Hahn2019Omega-regularLearning,Hasanbeig2020DeepLogics}, this novel product MDP incorporates a counter that counts the number of accepting states visited by paths starting at the initial state.

\begin{definition}[Product MDP]
\label{def:productMDP}
Given an MDP $\MDP=(\states,\state_0,\actions,\transitions,\propositions,\labFunc,\reward,\discount)$, an LDBA $\automaton=(\propositions,\autoStates,\autoState_0,\autoTransitions,\autoAccept)$ and $K\in\naturalNumber$, we construct the product MDP as follows: 
$$
\productMDP=\MDP\times\automaton\times[0..K]=(\productStates,\state_0^\times,\productActions,\productTransitions,\productAccept,\productReward,\productDiscount), 
$$
where the product states $\productStates=\states\times\autoStates\times[0..K]$, the initial state $\state_0^\times=(\state_0,\autoState_0,0)$, the product actions $\productActions=\actions\cup\{\epsilon_\autoState\mid\autoState\in\autoStates\}$, the accepting set $\productAccept=\states\times\autoAccept\times[0..K]$, and the product transitions $\productTransitions:\productStates\times\productActions\times\productStates \rightarrow [0,1]$, which are defined as:
\begin{align}
&\productTransitions((\state,\autoState,n),\action,(\hat{\state},\hat{\autoState},n))\\
&=\begin{cases}\transitions(\state,\action,\hat{\state}) &\text{ if } \action\in\actions\text{ and } \hat{\autoState}\in\autoTransitions(\autoState,\labFunc(\state))\setminus\autoAccept;\\
    1 &\text{ if }\action=\epsilon_{\hat{\autoState}}, \hat{\autoState}\in\autoTransitions(\autoState,\epsilon)\text{ and }\hat{\state}=\state;\\
    0 &\text{ otherwise }.
    \end{cases}\nonumber\\
    &\productTransitions((\state,\autoState,n),\action,(\hat{\state},\hat{\autoState},\min(n+1,K)))\\
    &=
    \begin{cases}\transitions(\state,\action,\hat{\state})&\text{ if }\action\in\actions\text{ and } \hat{\autoState}\in\autoTransitions(\autoState,\labFunc(\state))\cap\autoAccept;\\
    0 &\text{ otherwise }.
    \end{cases}\nonumber
\end{align}
where all other transitions are equal to 0. The product reward $\productReward:\productStates\times\productActions\times\productStates\rightarrow\realNumber$ and the product discount function $\productDiscount:\productStates\rightarrow(0,1]$ that are suitable for LTL learning are defined later in Definition~\ref{reward_defn}.

Furthermore, an infinite path $\MDPpath$ of $\productMDP$ satisfies the Büchi condition $\LTL_\autoAccept$ if $\textnormal{ inf}(\MDPpath)\cap\productAccept\neq\varnothing$. With a slight abuse of notation we denote this condition in LTL language as $\MDPpath\models\textsf{\upshape G}\textsf{\upshape F }\LTL_\autoAccept$, meaning for all $M\in\naturalNumber$, there always exists $\productState\in\productAccept$ that will be visited in $\MDPpath[M:]$.
\end{definition}

When an MDP action $\action\in\actions$ is taken in the product MDP $\productMDP$, the alphabet used to transition the LDBA is deduced by applying the label function to the current environment MDP state: $\labFunc(\state)\in2^{\propositions}$. In this case, the LDBA transition $\delta(\autoState, \labFunc(\state))$ is deterministic. Otherwise, if an $\epsilon$-action $\epsilon_{\hat{\autoState}}\in\{\epsilon_\autoState\mid\autoState\in\autoStates\}$ is taken, LDBA is transitioned with an $\epsilon$-transition, and the non-determinism of $\delta(q, \epsilon)$ is resolved by transitioning the automaton state to $\hat{q}$. The $K$ counter value is equal to 0 in the initial state, and each time an accepting state is reached, the counter value increases by one until it is capped at $K$.

\begin{example}
\label{example_mdp}
We motivate our product MDP structure of Definition~\ref{def:productMDP} through an example. In Figure~\ref{fig:example_mdp}, we have a grid environment where the agent can decide to go up, down, left or right. The task is to visit states labeled \enquote{a} infinitely often without visiting \enquote{c} as described by the LDBA in Figure~\ref{fig:LDBA}. The MDP starts at (1,0), with walls denoted by solid gray squares. The states in the first row only allow action right as denoted by the right pointing triangle, which leads to a sink at (0,9). There is also a probabilistic gate at (2,0) that transitions the agent randomly to go down or right, and if the agent reaches (2,1), the accepting sink at (2,9) is reachable. Therefore, the optimal policy is to go down from the start and stay in (2,9) if the probabilistic gate at (2,0) transitions the agent to the right. The probability of satisfying this task is the probability of the gate sending you to the right. Intuitively, this environment has some initial accepting states that are easy to explore, but lead to non-accepting sinks, whereas the true optimal path requires more exploration. If we set $K=10$ in the product MDP for this task, we can assign very small rewards for the initially visited accepting states and gradually increase the reward as more accepting states are visited to encourage exploration and guide the agent to the optimal policy.
\end{example}

Next, we provide a theorem, which states that the product MDP with Büchi condition $\LTL_\autoAccept$ is equivalent, in terms of the optimal policy, to the original MDP with LTL specification $\LTL$. The proof of this theorem is provided in Appendix A.1.

\begin{theorem}[Satisfiability Equivalence]
For any product MDP $\productMDP$ that is induced from LTL formula $\LTL$, we have that
\begin{equation}
\sup_\policy{\probP_\policy(\state_0\models\LTL)}=\sup_{\policy^\times}{\probP_{\policy^\times}(\productState_0\models\textsf{\upshape G}\textsf{\upshape F }\LTL_\autoAccept)}.
\end{equation}
Furthermore, a deterministic memoryless policy that maximizes the probability of satisfying the Büchi condition $\LTL_\autoAccept$ on the product MDP $\productMDP$, starting from the initial state, induces a deterministic finite-memory optimal policy that maximizes the probability of satisfying $\LTL$ on the original MDP $\MDP$ from the initial state.
\label{thm:equivalence}
\end{theorem}

\subsubsection{Reward Structure for LTL learning}
We first define a generalized reward structure in $\productMDP$ for LTL learning, and then prove the equivalence between acquiring the highest expected discounted reward and achieving the highest probability of satisfying $\LTL_\autoAccept$ under this reward structure.

\begin{definition}[Reward Structure]
\label{reward_defn}
Given a product MDP $\productMDP$ and a policy $\policy$, the product reward function $\productReward:\productStates\times\productActions\times\productStates\rightarrow\realNumber$ is suitable for LTL learning if 
\begin{align}
\productReward((\state,\autoState,n),\productAction,(\state^\prime,\autoState^\prime,n^\prime))=
    \begin{cases} 
    R_n &\text{ if }\autoState^\prime\in\autoAccept;\\
    0  &\text{ otherwise,}
    \end{cases}
\end{align}
where $R_n\in(0,U]$ are constants for $n\in[0..K]$ and $U\in(0,1]$ is an upper bound on the rewards. The rewards are non-zero only for accepting automaton states, and depend on the value of the $K$ counter.

Then, given a discount factor $\discount\in(0,1]$, we define the product discount function $\productDiscount: \productStates\rightarrow(0,1]$ as $$\productDiscount(\productState_j)=\begin{cases}
1-\productReward_{j} &\text{ if } \productReward_{j}>0;\\
\discount &\text{ otherwise,}
\end{cases}$$
and the expected discounted reward following policy $\policy$ starting at $\productState$ and time step $t$ is
\begin{align}
    \resizebox{.95\hsize}{!}{$\ExpReward^\policy_t(\productState)=\expectE_\policy[\sum\limits^\infty_{i=t}(\prod\limits_{j=t}^{i-1}\productDiscount(\productState_j))\cdot\productReward(\productState_{i},\productAction_{i},\productState_{i+1})\mid \productState_t=\productState]$}.
\end{align}
\end{definition}

The highest $K$ value reached in a path (i.e., the number of accepting states visited in the path) acts as a measure of how promising that path is for satisfying $\LTL_\autoAccept$. By exploiting it, we can assign varying rewards to accepting states to guide the agent, as discussed in the motivating example in Section~\ref{sec:product_MDP}. Next, we provide a lemma stating the properties of the product MDP regarding the satisfaction of the Büchi condition $\LTL_\autoAccept$.

\begin{lemma}
\label{useful_lemma}
Given a product MDP $\productMDP$ with its corresponding LTL formula $\LTL$ and a policy $\policy$, we write $\productMDP_\policy$ for the induced Markov chain from $\policy$. Let $B_\productAccept$ denote the set of states that belong to accepting BSCCs of $\productMDP_\policy$, and $B^\times_\varnothing$ denote the set of states that belong to rejecting BSCCs:
\begin{align}
&\resizebox{.87\hsize}{!}{$B_\productAccept \defeq \{\productState \mid \productState\in B \in BSCC(\productMDP_\policy),B\cap\productAccept\neq\varnothing\}$};\\
&\resizebox{.87\hsize}{!}{$B^\times_\varnothing \defeq \{\productState \mid \productState\in B \in BSCC(\productMDP_\policy),B\cap\productAccept=\varnothing\}$},
\end{align}
where $BSCC(\productMDP_\policy)$ is the set of all BSCCs of $\productMDP_\policy$. We further define more general accepting and rejecting sets:
\begin{align}
    B_\autoAccept &\defeq \{(\state,\autoState,n) \mid \exists n^\prime\in [0..K] : (\state,\autoState,n^\prime)\in B_\productAccept\}\label{B_accept};\\
    B_\varnothing &\defeq \{(\state,\autoState,n) \mid \exists n^\prime\in [0..K] : (\state,\autoState,n^\prime)\in B^\times_\varnothing\label{B_reject}\}.
\end{align}
We then have that $\probP_\policy(\productState\models\textsf{\upshape G}\textsf{\upshape F }\LTL_\autoAccept)=1 \;\forall \productState\in B_\autoAccept$, $\probP_\policy(\productState\models\textsf{\upshape G}\textsf{\upshape F }\LTL_\autoAccept)=0\;\forall\productState\in B_\varnothing$ and $B_\varnothing\cap\productAccept=\varnothing$. Furthermore, $B_\autoAccept$ and $B_\varnothing$ are sink sets, meaning once the set is reached, no states outside the set can be reached.
\end{lemma}

The proof of this lemma is provided in Appendix A.2. Using this lemma, we can now state and proof the main theorem of this paper.

\begin{theorem}[Optimality guarantee]
Given an LTL formula $\LTL$ and a product MDP $\productMDP$, there exists an upper bound $U\in(0,1]$ for rewards and a discount factor $\discount\in(0,1]$ such that for all product rewards $\productReward$ and product discount functions $\productDiscount$ satisfying Definition~\ref{reward_defn}, the optimal deterministic memoryless policy $\policy_\reward$ that maximizes the expected discounted reward $G^{\policy_\reward}_0(\productState_0)$ is also an optimal policy $\policy_\LTL$ that maximizes the probability of satisfying the Büchi condition $\probP_{\policy_\LTL}(\productState_0\models\textsf{\upshape G}\textsf{\upshape F }\LTL_\autoAccept)$ on the product MDP $\productMDP$.
\label{thm:optimality}
\end{theorem}

\begin{proof}[Proof sketch]
We now present a sketch of the proof to provide intuition for the main steps and the selection of key parameters. The full proof is provided in Appendix A.3.

To ensure optimality, given a policy $\policy$ with the product MDP $\productMDP$ and the LTL formula $\LTL$, we want to demonstrate a tight bound between the expected discounted reward following $\policy$ and the probability of $\policy$ satisfying $\LTL$, such that maximizing one quantity is equivalent to maximizing the other.

At a high level, we want to select the two key parameters, the reward upper bound $U\in(0,1]$ and the discount factor $\discount\in(0,1]$, to adequately bound: (i) the rewards given for paths that eventually reach rejecting BSCCs (thus not satisfying the LTL specification); and (ii) the discount of rewards received from rejecting states for paths that eventually reach accepting BSCCs. 

We informally denote by $C^\policy_\varnothing$ the expected number of visits to accepting states before reaching a rejecting BSCC, and (i) can be sufficiently bounded by selecting $U=1/C^\policy_\varnothing$. Next, we informally write $C^\policy_\autoAccept$ for the expected number of rejecting states visited before reaching an accepting BSCC, and denote by $N^\policy$ the expected steps between visits of accepting states in the accepting BSCC. Intuitively, for (ii), we bound the amount of discount before reaching the accepting BSCC using $C^\policy_\autoAccept$, and we bound the discount after reaching the BSCC using $N^\policy$, yielding $\discount=1-1/(C^\policy_\varnothing*N^\policy+C^\policy_\autoAccept)$. 

In practice, using upper bounds of $C^\policy_\varnothing, C^\policy_\autoAccept$ and $N^\policy$ instead also ensures optimality, and those bounds can be deduced from assumptions about, or knowledge of, the MDP.
\end{proof}

As shown in the proof sketch, selecting $U=1/C^\policy_\varnothing$ and $\discount=1-1/(C^\policy_\varnothing*N^\policy+C^\policy_\autoAccept)$ is sufficient to ensure optimality. Using the example of the probabilistic gate MDP in Figure~\ref{fig:example_mdp}, we have that $C^\policy_\varnothing\approx C^\policy_\autoAccept\leq 10$ and $N^\policy=1$, so choosing $U=0.1$ and $\gamma=0.95$ is sufficient to guarantee optimality. For more general MDPs, under the common assumption that the number of states $\abs{\states}$ and the minimum non-zero transition probability $p_{\min}\coloneqq \min_{s,a,s^\prime}\{\transitions(\state,\action,\state^\prime)>0\}$ are known, $C^\policy_\varnothing$ and $C^\policy_\autoAccept$ can be upper bounded by $\abs{\states}/p_{\min}$, while $N^\policy$ can be upper bounded by $\abs{\states}$.

\subsection{LTL learning with Q-learning}

\begin{algorithm}[tb]
\DontPrintSemicolon
\SetNlSty{textbf}{}{:}
\caption{KC Q-learning from LTL}
\label{alg:Q_learning}
\BlankLine
\SetKw{KwIn}{in}
\SetKwInOut{Input}{Input}
\Input{environment MDP $\MDP$, LTL formula $\LTL$}
  {
  \nl translate $\LTL$ into an LBDA $\automaton$\\
  \nl construct product MDP $\productMDP$ using $\MDP$ and $\automaton$\\
  \nl initialize Q for each $\state$ and $\autoState$ pair\\
  \nl \For {$l\gets 0$ to max\textunderscore episode}{
  \nl initialize $(\state,\autoState,n)\gets(\state_0,\autoState_0,0)$\\
  \nl \For{$t\gets 0$ to max\textunderscore timestep}{
  \nl get policy $\policy$ derived from Q (\textit{e.g.,} $\epsilon$-greedy)\\
  \nl take action $\productAction\gets\policy((\state,\autoState,n))$ in $\productMDP$\\
  \nl get next product state $(\state^\prime,\autoState^\prime,n^\prime)$\\
  \nl $r\gets\productReward((\state,\autoState,n),\productAction,(\state^\prime,\autoState^\prime,n^\prime))$\\
 \nl $\discount\gets\discount^\prime(\state,\autoState,n)$ \\
  \nl $Q((\state,\autoState),\productAction)\overset{\alpha}{\gets}r+\discount\max_{\action\in\productActions}Q((\state^\prime,\autoState^\prime),\action)$\\
  \nl   update $(\state,\autoState,n)\gets(\state^\prime,\autoState^\prime,n^\prime)$
        }
     }
  \nl   gets greedy policy $\policy_\LTL$ from Q\\
  \nl   \Return induced policy on $\MDP$ by removing $\epsilon$-actions
  }
\end{algorithm}

Employing this product MDP $\productMDP$ and its reward structure, we present Algorithm~\ref{alg:Q_learning} (KC), a model-free Q-learning algorithm for LTL specifications utilizing the K counter product MDP. The product MDP is constructed on the fly as we explore: for action $\productAction\in\actions$, observe the next environment state $\state^\prime$ by taking action $\productAction$ in environment state $\state$. Then, we compute the next automaton state $\autoState^\prime$ using transition function $\autoTransitions(\autoState,\labFunc(\state))$ and counter $n$ depending on whether $n\leq K$ and $\autoState^\prime\in\autoAccept$. If $\productAction\in\{\epsilon_\autoState\mid\autoState\in\autoStates\}$, update $\autoState^\prime$ using the $\epsilon$-transition and leave environment state $\state$ and counter $n$ unchanged. However, directly adopting Q-learning on this product MDP yields a Q function defined on the whole product state space $\productStates$, meaning the agent needs to learn the Q function for each $K$ value. To improve efficiency, we propose to define the Q function on the environment states $\states$ and automaton states $\autoStates$ only, and for a path $\MDPpath$ of $\productMDP$, the update rule for the Q function at time step $t$ is:
\begin{align}
Q_{t+1}((\state_t,\autoState_t),\productAction_t)&\overset{\alpha}{\gets}\productReward(\productState_t,\productAction_t,\productState_{t+1})\\
&+\discount^\prime(\productState_{t})\max_{\productAction\in\productActions}Q_t((\state_{t+1},\autoState_{t+1}),\productAction)\nonumber,
\end{align}
where $\productState_t=(\state_t,\autoState_t,n_t)$, $\productAction_t$ is the action taken at time step $t$, ${\productState_{t+1}=(\state_{t+1},\autoState_{t+1},n_{t+1})}$ is the next product state, and $\alpha$ is the learning rate. We claim that, with this collapsed Q function, the algorithm returns the optimal policy for satisfying $\LTL$ because the optimal policy is independent from the K counter, with the proof provided in Appendix A.4.
\begin{theorem}
\label{thm:q_learning}
Given an environment MDP $\MDP$ and an LTL specification $\LTL$ with appropriate discount factor $\gamma$ and reward function $\productReward$ satisfying Theorem~\ref{thm:optimality}, Q-learning for LTL described in Algorithm~\ref{alg:Q_learning} converges to an optimal policy $\policy_\LTL$ that maximizes the probability of satisfying $\LTL$ on $\MDP$.
\end{theorem}

\subsection{Counterfactual Imagining}
\begin{algorithm}[tb]
\DontPrintSemicolon
\SetNlSty{textbf}{}{:}
\caption{CF+KC Q-learning from LTL}
\label{alg:CounterQ}
\BlankLine
\SetKw{KwIn}{in}
\SetKwInOut{Input}{Input}
\Input{environment MDP $\MDP$, LTL formula $\LTL$}
  {
  \nl translate $\LTL$ into an LBDA $\automaton$\\
  \nl construct product MDP $\productMDP$ using $\MDP$ and $\automaton$\\
  \nl initialize Q for each $\state$ and $\autoState$ pair\\
  \nl \For {$l\gets 0$ to max\textunderscore episode}{
  \nl initialize $(\state,\autoState,n)\gets(\state_0,\autoState_0,0)$\\
  \nl \For{$t\gets 0$ to max\textunderscore timestep}{
  \nl get policy $\policy$ derived from Q (e.g., $\epsilon$-greedy)\\
  \nl get action $\productAction\gets\policy((\state,\autoState,n))$ in $\productMDP$\\
  \nl \For {$\Bar{\autoState}\in\autoStates$}{
  \nl  get counterfactual imagination $(\state^\prime,\Bar{\autoState}^\prime,n^\prime)$ by taking action $\productAction$ at $(\state,\Bar{\autoState},n)$\\
  \nl $r\gets\productReward((\state,\Bar{\autoState},n),\productAction,(\state^\prime,\Bar{\autoState}^\prime,,n^\prime))$\\
  \nl $\discount\gets\discount^\prime(\state,\Bar{\autoState},n)$ \\
  \nl $Q((\state,\Bar{\autoState}),\productAction)\overset{\alpha}{\gets}r+\discount\max_{\action\in\productActions}Q((\state^\prime,\Bar{\autoState}^\prime),\action)$\\
  }
  \nl
  obtain $(\state^\prime,\autoState^\prime,n^\prime)$ using action $\productAction$ at $(\state,\autoState,n)$\\
  \nl   update $(\state,\autoState,n)\gets(\state^\prime,\autoState^\prime,n^\prime)$
        }
     }
  \nl   gets greedy policy $\policy_\LTL$ from Q\\
  \nl   \Return induced policy on $\MDP$ by removing $\epsilon$-actions
  }
\end{algorithm}

\begin{figure*}[t]
\centering
\begin{subfigure}[t]{0.33\textwidth}
    \centering
    \includegraphics[width=\textwidth]{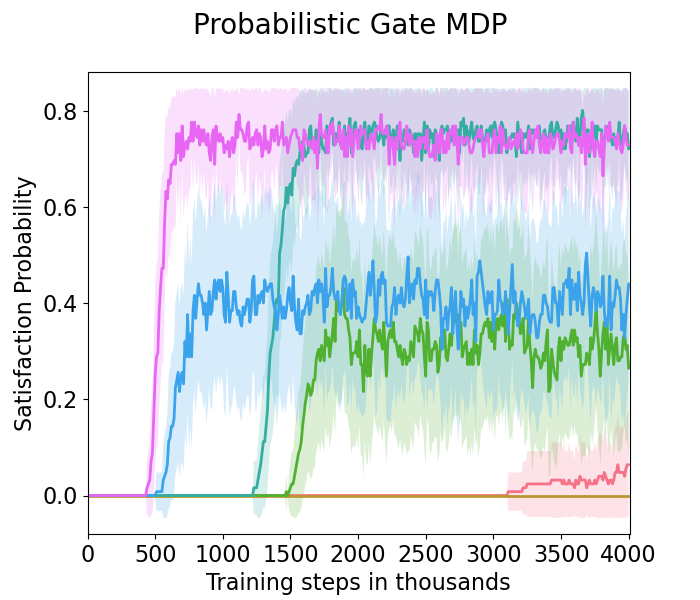}
    \label{fig:hard1_result}
\end{subfigure}
\begin{subfigure}[t]{0.33\textwidth}
    \centering
    \includegraphics[width=\textwidth]{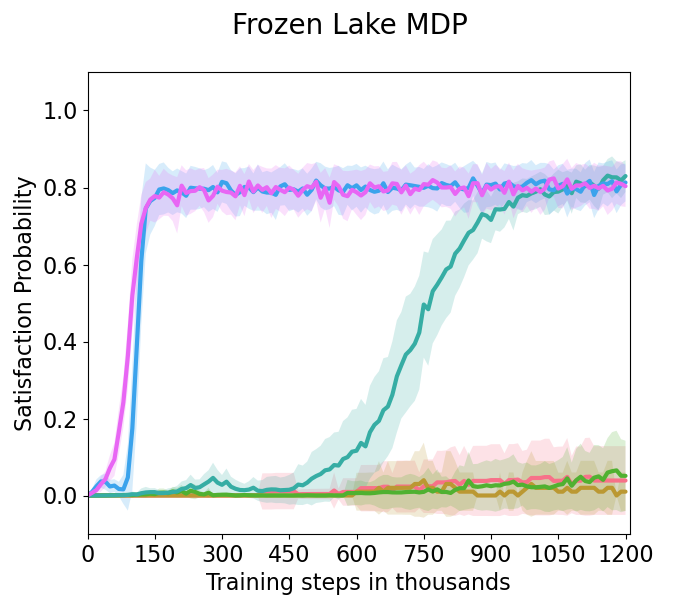}
    \label{fig:frozen_result}
\end{subfigure}
\begin{subfigure}[t]{0.33\textwidth}
    \centering
    \includegraphics[width=\textwidth]{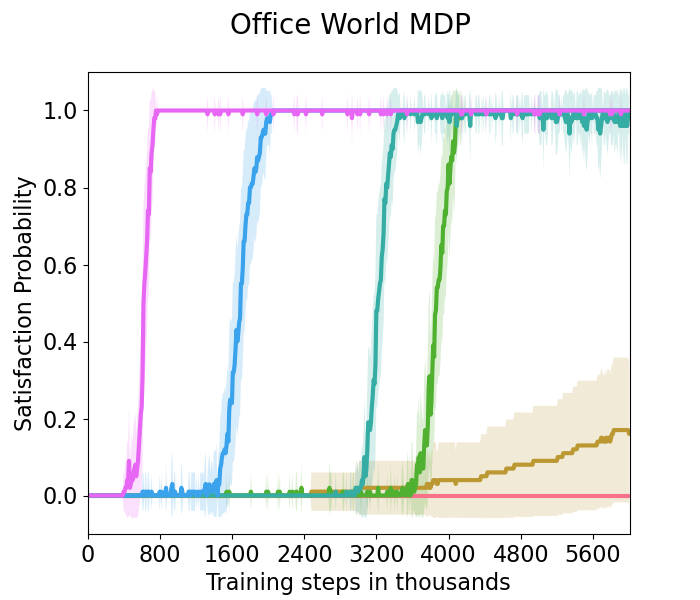}
    \label{fig:office_result}
\end{subfigure}
\begin{subfigure}[t]{\textwidth}
    \centering
\includegraphics[width=0.8\textwidth]{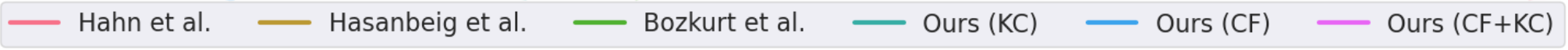}
\end{subfigure}
\caption{The experimental results on various MDPs and tasks.}
\label{fig:all_results}
\end{figure*}
Additionally, we propose a method to exploit the structure of the product MDP, specifically the LDBA, to facilitate learning. We use counterfactual reasoning to generate synthetic imaginations: from one state in the environment MDP, imagine we are at each of the automaton states while taking the same actions.

If the agent is at product state $(\state,\autoState,n)$ and an action $\productAction\in\actions$ is chosen, for each $\Bar{\autoState}\in\autoStates$, the next state $(\state^\prime,\Bar{\autoState}^\prime,n^\prime)$ by taking action $\productAction$ from $(\state,\Bar{\autoState},n)$ can be computed by first taking the action in environment state, and then computing the next automaton state $\Bar{\autoState}^\prime=\autoTransitions(\Bar{\autoState},\labFunc(\state))$ and the next $K$ value $n^\prime$. The reward for the agent is $\productReward(\state,\Bar{\autoState},\Bar{n})$, and we can therefore update the Q function with this enriched set of experiences. These experiences produced by counterfactual imagining are still sampled from $\productMDP$ following the transition function $\productTransitions$, and hence, when used in conjunction with any off-policy learning algorithms like Q-learning, the optimality guarantees of the algorithm are preserved.

As shown in Algorithm~\ref{alg:CounterQ} (CF-KC), counterfactual imagining (CF) can be incorporated into our KC Q-learning algorithm by altering a few lines (line 9-12 in Algorithm~\ref{alg:CounterQ}) of code, and it can also be used in combination with other automata product MDP RL algorithms for LTL. Note that the idea of counterfactual imagining is similar to that proposed by Toro Icarte \textit{et al.}~\shortcite{Icarte2022RewardLearning}, but our approach has adopted LDBAs in the product MDPs for LTL specification learning.

\section{Experimental Results}

We evaluate our algorithms on various MDP environments, including the more realistic and challenging stochastic MDP environments\footnote{The implementation of our algorithms and experiments can be found on GitHub: \url{https://github.com/shaodaqian/rl-from-ltl}}. We propose a method to directly evaluate the probability of satisfying LTL specifications by employing probabilistic model checker PRISM~\cite{Kwiatkowska2011PRISMSystems}. We build the induced MC $\MDP_\policy$ from the environment MDP and the policy in PRISM format, and adopt PRISM to compute the exact satisfaction probability of the given LTL specification. We utilize tabular Q-learning as the core off-policy learning method to implement our three algorithms: Q-learning with $K$ counter reward structure (KC), Q-learning with $K$ counter reward structure and counterfactual imagining (CF+KC), and Q-learning with only counterfactual imagining (CF), in which we set $K=0$. We compare the performance of our methods against the methods proposed by Bozkurt \textit{et al.}~\shortcite{Bozkurt2019ControlLearning}, Hahn \textit{et al.}~\shortcite{Hahn2019Omega-regularLearning} and Hasanbeig \textit{et al.}~\shortcite{Hasanbeig2020DeepLogics}. Note that our KC algorithm, in the special case that $K=0$ with no counterfactual imagining, is algorithmically equivalent to Bozkurt \textit{et al.}~\shortcite{Bozkurt2019ControlLearning} when setting their parameter $\gamma_B=1-U$. Our methods differ from all other existing methods to the best of our knowledge. The details and setup of the experiments are given in Appendix B.

We set the learning rate $\alpha=0.1$ and $\epsilon=0.1$ for exploration. We also set a relatively loose upper bound on rewards $U=0.1$ and discount factor $\gamma=0.99$ for all experiments to ensure optimality. Note that the optimality of our algorithms holds for a family of reward structures defined in Definition~\ref{reward_defn}, and for experiments we opt for a specific reward function that linearly increases the reward for accepting states as the value of $K$ increases, namely $r_n=U\cdot n/K \; \forall n\in[0..K]$, to facilitate training and exploration. The Q function is optimistically initialized by setting the Q value for all available state-action pairs to $2U$. All experiments are run 100 times, where we plot the average satisfaction probability 
with half standard deviation in the shaded area.

First, we conduct experiments on the probabilistic gate MDP described in Example~\ref{example_mdp} with task \enquote{\textsf{\upshape F}\textsf{\upshape G} a \& \textsf{\upshape G} !c}, which means reaching only states labeled \enquote{a} in the future while never reaching \enquote{c} labeled states. We set $K=10$ for this task, and in Figure~\ref{fig:all_results} (left), compared to the other three methods, our method KC achieved better sample efficiency and convergence and CF demonstrates better sample efficiency while still lacking training stability. The best performance is achieved by CF+KC, while other methods either exhibit slower convergence (Bozkurt \textit{et al.}~\shortcite{Bozkurt2019ControlLearning} and Hahn \textit{et al.}~\shortcite{Hahn2019Omega-regularLearning}) or fail to converge (Hasanbeig \textit{et al.}~\shortcite{Hasanbeig2020DeepLogics}) due to the lack of theoretical optimality guarantees.


The second MDP environment is the $8\times8$ frozen lake environment from OpenAI Gym~\cite{Brockman2016OpenAIGym}. This environment consists of frozen lake tiles, where the agent has 1/3 chance of moving in the intended direction and 1/3 of going sideways each, with details provided in Appendix B.2. The task is \enquote{(\textsf{\upshape G}\textsf{\upshape F} a $\mid$ \textsf{\upshape G}\textsf{\upshape F} b) \& \textsf{\upshape G} !h}, meaning to always reach lake camp \enquote{a} or lake camp \enquote{b} while never falling into holes \enquote{h}. We set $K=10$ for this task, and in Figure~\ref{fig:all_results} (middle), we observe significantly better sample efficiency for all our methods, especially for CF+KC and CF, which converge to the optimal policy at around 150k training steps. The other three methods, on the other hand, barely start to converge at 1200k training steps. CF performs especially well in this task because the choice of always reaching \enquote{a} or \enquote{b} can be considered simultaneously during each time step, reducing the sample complexity to explore the environment.

Lastly, we experiment on a slight modification of the more challenging office world environment proposed by Toro Icarte \textit{et al.}~\shortcite{Icarte2022RewardLearning}, with details provided in Appendix B.3. We include patches of icy surfaces in the office world, with the task to either patrol in the corridor between \enquote{a} and \enquote{b}, or write letters at \enquote{l} and then patrol between getting tea \enquote{t} and workplace \enquote{w}, while never hitting obstacles \enquote{o}. $K=5$ is set for this task for a steeper increase in reward, since the long distance between patrolling states makes visiting many accepting states in each episode time consuming. Figure~\ref{fig:all_results} (right) presents again the performance benefit of our methods, with CF+KC performing the best and CF and KC second and third, respectively. For this challenging task in a large environment, the method of Hahn \textit{et al.}~\shortcite{Hahn2019Omega-regularLearning} requires the highest number of training steps to converge.

Overall, the results demonstrate KC improves both sample efficiency and training stability, especially for challenging tasks. In addition, CF greatly improves sample efficiency, which, combined with KC, achieves the best results.

\subsection{Runtime analysis and sensitivity analysis}

\begin{figure}[tb]
    \vspace{-0.5cm}
    \centering
    \includegraphics[width=0.4\textwidth]{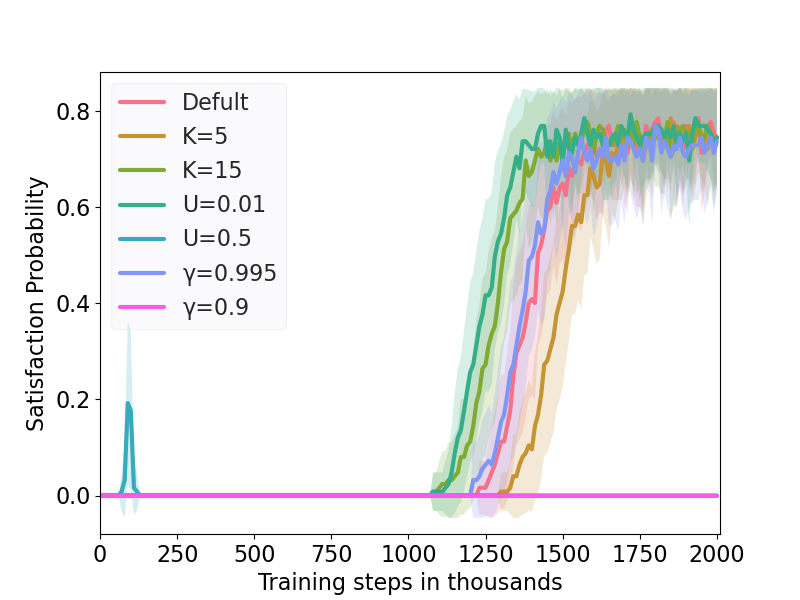}
    \caption{Sensitivity analysis on key parameters.}
    \label{fig:sensitivity_analysis}
\end{figure}

It is worth mentioning that, for counterfactual imagining, multiple updates to the Q function are performed at each step in the environment. This increases the computational complexity, but the additional inner loop updates on the Q function will only marginally affect the overall computation time if the environment steps are computationally expensive. Taking the office world task as an example, the average time to perform 6 million training steps are 170.9s, 206.3s and 236.1s for KC, CF and CF+KC, respectively. However, the time until convergence to the optimal policies are 96.8s, 70.2s and 27.5s for KC, CF and CF+KC, respectively.

For sensitivity analysis on the key parameters, we run experiments on the probabilistic gate MDP task with different parameters against the default values of $U=0.1, \gamma=0.99$ and $K=10$. As shown in Figure~\ref{fig:sensitivity_analysis}, if $U(=0.5)$ is chosen too high or $\gamma(=0.9)$ is chosen too low, the algorithm does not converge to the optimal policy as expected. However, looser parameters $U=0.01$ and $\gamma=0.995$ do not harm the performance, which means that, even with limited knowledge of the underlying MDP, our algorithm still performs well with loose parameters. Optimality is not affected by the $K$ value, while the performance is only mildly affected by different $K$ values.

\section{Conclusion}

We presented a novel model-free reinforcement learning algorithm to learn the optimal policy of satisfying LTL specifications in an unknown stochastic MDP with optimality guarantees. We proposed a novel product MDP, a generalized reward structure and a RL algorithm that ensures convergence to the optimal policy with the appropriate parameters. Furthermore, we incorporated counterfactual imagining, which exploits the LTL specification to create imagination experiences. Lastly, utilizing PRISM~\cite{Kwiatkowska2011PRISMSystems}, we directly evaluated the performance of our methods and demonstrated superior performance on various MDP environments and LTL tasks.

Future works include exploring other specific reward functions under our generalized reward structure framework, tackling the open problem~\cite{Alur2021ALearning} of dropping all assumptions regarding the underlying MDP, and extending the theoretical framework to continuous states and action environment MDPs, which might be addressed through an abstraction of the state space. In addition, utilizing potential-based reward shaping~\cite{Ng1999PolicyShaping}\cite{Devlin2012DynamicShaping} to exploit the semantic class structures of LTL specifications and transferring similar temporal logic knowledge of the agent~\cite{Xu2019TransferLearning} between environments could also be interesting.

\section*{Acknowledgments}
This work was supported by the EPSRC Prosperity Partnership FAIR (grant number EP/V056883/1). DS acknowledges funding from the Turing Institute and Accenture collaboration.  
MK receives funding from the ERC under the European Union’s Horizon 2020 research and innovation programme (\href{http://www.fun2model.org}{FUN2MODEL}, grant agreement No.~834115).

\bibliographystyle{named}
\bibliography{ijcai23}

\appendix
\include{appendix}

\end{document}

%% file: macro.tex
\newcommand{\realNumber}{\mathbb{R}}
\newcommand{\naturalNumber}{\mathbb{N}}
\newcommand{\integerNumber}{\mathbb{Z}}
\newcommand{\probP}{\mathds{P}}
\newcommand{\expectE}{\mathds{E}}

\newcommand{\vertices}{V}
\newcommand{\vertex}{v}
\newcommand{\edges}{E}
\newcommand{\edge}{e}

\newcommand{\MDP}{\mathcal{M}}
\newcommand{\states}{S}
\newcommand{\state}{s}
\newcommand{\actions}{A}
\newcommand{\action}{a}
\newcommand{\transitions}{T}
\newcommand{\reward}{r}
\newcommand{\ExpReward}{G}
\newcommand{\labFunc}{L}
\newcommand{\propositions}{\mathcal{AP}}
\newcommand{\alphabets}{\Sigma}
\newcommand{\alphabet}{\nu}
\newcommand{\policy}{\pi}
\newcommand{\discount}{\gamma}
\newcommand{\Value}{v}
\newcommand{\paths}{Paths^\MDP}
\newcommand{\Fpaths}{FPaths^\MDP}
\newcommand{\filtration}{\mathcal{F}}
\newcommand{\MDPpath}{\sigma}

\newcommand{\productMDP}{\mathcal{M}^\times}
\newcommand{\productStates}{S^\times}
\newcommand{\productState}{s^{\times}}

\newcommand{\productTransitions}{T^\times}
\newcommand{\productAccept}{\mathcal{F}^\times}
\newcommand{\productActions}{A^\times}
\newcommand{\productAction}{a^\times}
\newcommand{\productReward}{\reward^\times}
\newcommand{\productDiscount}{\gamma^\times}

\newcommand{\automaton}{\mathcal{A}}
\newcommand{\autoStates}{\mathcal{Q}}
\newcommand{\autoState}{\mathpzc{q}}
\newcommand{\autoTransitions}{\Delta}
\newcommand{\autoAccept}{\mathcal{F}}
\newcommand{\LTL}{\varphi}
\newcommand{\Autorun}{\theta}
\newcommand{\word}{\mathcal{W}}
\newcommand{\potential}{\Phi}

\newcommand{\defeq}{\vcentcolon=}
\newcommand{\eqdef}{=\vcentcolon}

\def\mcirc{\mathbin{\scalerel*{\circ}{j}}}

%% file: appendix.tex
\onecolumn
\section{Theoretical Results: Proofs}
\subsection{Proof of Theorem~\ref{thm:equivalence}}
\label{proof:equivalence}
\begin{proof}

We will prove the equality by verifying both sides of the inequality and subsequently constructing the induced policy on $\MDP$.

For $\leq$, it has been proven~\cite{Sickert2016Limit-deterministicLogic} that, for every accepting path of the environment MDP $\MDP$ following a policy $\policy$, there always exists a corresponding accepting run $\Autorun$ of the LDBA that resolves the non-determinism. Augmenting the resolved non-determinism as $\epsilon$-actions to the original path yields an accepting path of $\productMDP$, since the counter will eventually reach $K$ and the accepting states $\productAccept$ will be visited infinitely often. Therefore, we have created a policy for $\productMDP$ that is at least as good as $\policy$ on $\MDP$.

For $\geq$, it is clear that any policy $\policy$ on $\productMDP$ can induce a policy for $\MDP$ by eliminating the $\epsilon$-transitions and removing the projection of the automaton and $K$ counter. Therefore, any path following $\policy$ that meets the Büchi condition $\LTL_\autoAccept$ will induce a path of $\MDP$ that is accepting by $\automaton$ induced from $\LTL$, where the non-determinism of $\automaton$ is resolved by $\epsilon$-transitions of $\policy$, thus satisfying $\LTL$.

\end{proof}

\subsection{Proof of Lemma~\ref{useful_lemma}}
\label{proof:useful_lemma}

\begin{proof}
To begin with, we recall the property of BSCCs in Markov chains (MC): for any infinite path $\MDPpath$ of a MC, a BSCC will eventually be reached, and once reached it can't reach any state outside the BSCC and all states within it will be visited infinitely often with probability 1. Therefore, since all $\productState\in B_\productAccept$ belongs to BSCCs with accepting states, all paths from $\productState$ will reach accepting states infinitely often with probability 1, so $\probP_\policy(\productState\models\textsf{\upshape G}\textsf{\upshape F }\LTL_\autoAccept)=1 \;\forall \productState\in B_\productAccept$. Similarly, all $\productState\in B^\times_\varnothing$ belong to BSCCs with no accepting states, so $\probP_\policy(\productState\models\textsf{\upshape G}\textsf{\upshape F }\LTL_\autoAccept)=0 \;\forall \productState\in B^\times_\varnothing$.


Now, we observe that, for states in $\productMDP_\policy$, the Markov chain transitions for environment states and automaton states are in fact independent of the counter value:
\begin{align}
&\productTransitions_\policy((\state,\autoState,n_1),(\state^\prime,\autoState^\prime,\min(n_1+1,K)))=\productTransitions_\policy((\state,\autoState,n_2),(\state^\prime,\autoState^\prime,\min(n_2+1,K)))\label{eq_lemma}
\end{align} 
for all $n_1,n_2\in[0..K], \state,\state^\prime\in\states$ and $\autoState,\autoState^\prime\in\autoStates$. With this observation, for all $(\state,\autoState,n)\in B_\autoAccept$, there exists $(\state,\autoState,n^\prime)\in B_\productAccept$, which belongs to some BSCC by Line~\ref{B_accept}. Therefore, any state $(\state_1,\autoState_1,n_1)$ reachable from $(\state,\autoState,n)$ by Equation~\ref{eq_lemma} implies there exists $n_2$ such that $(\state_1,\autoState_1,n_2)$ is reachable from $(\state,\autoState,n^\prime)$. This,  by the definition of BSCC, means $(\state_1,\autoState_1,n_2)\in B_\productAccept$, which implies $(\state_1,\autoState_1,n_1)\in B_\autoAccept$ by Line~\ref{B_accept}, which shows that $B_\autoAccept$ is a sink set. Furthermore, since $(\state,\autoState,n^\prime)\in B_\productAccept$ belongs to an accepting BSCC, it can reach an accepting state $(\state_\autoAccept,\autoState_\autoAccept,n_\autoAccept)\in B_\productAccept$ in that BSCC with probability 1. By Equation~\ref{eq_lemma}, there exists $(\state_\autoAccept,\autoState_\autoAccept,n_3)\in B_\autoAccept$ reachable from $(\state,\autoState,n)$ with probability 1 and, by the accepting states of Definition~\ref{def:productMDP} of product MDP, $(\state_\autoAccept,\autoState_\autoAccept,n_3)$ is also accepting, meaning accepting states can be reached from any state in $B_\autoAccept$ with probability 1. Together with the fact that $B_\autoAccept$ is a sink set, we conclude that accepting states can be reached infinitely often with probability 1 from $B_\autoAccept$, meaning $\probP_\policy(\productState\models\textsf{\upshape G}\textsf{\upshape F }\LTL_\autoAccept)=1 \;\forall \productState\in B_\autoAccept$. Another observation is that, since the counter values are increasing for all paths and accepting states will be reached infinitely often in accepting BSCCs, all states in accepting BSCCs must have counter value equal to $K$.

Using the same argument, we conclude that $B_\varnothing$ is also a sink set. In addition, for $(\state,\autoState,n)\in B_\varnothing$, there exists $(\state,\autoState,n^\prime)\in B^\times_\varnothing$ which belongs to a rejecting BSCC such that no accepting states can be reached from it. By Equation~\ref{eq_lemma} and the accepting states definition of product MDP~(Definition~\ref{def:productMDP}), we see that no accepting state can be reached from $(\state,\autoState,n)$ either, which implies $\probP_\policy(\productState\models\textsf{\upshape G}\textsf{\upshape F }\LTL_\autoAccept)=0\;\forall\productState\in B_\varnothing$ and $B_\varnothing\cap\productAccept=\varnothing$, which completes the proof.
\end{proof}

\subsection{Proof of Theorem~\ref{thm:optimality}}
\label{proof:optimality}

\begin{proof}
To give an outline of the proof, we would like to show that, for any deterministic memoryless policy $\policy$, the expected discounted reward  $G^\policy_0(\productState)$ for all product states with counter value 0, \textit{i.e.} $\productState=(\state,\autoState,0)$, is close to the probability $\probP_\policy(\productState\models\textsf{\upshape G}\textsf{\upshape F }\LTL_\autoAccept)$ of satisfying $\LTL_\autoAccept$ starting from $\productState$ following $\policy$. We show this by upper and lower bounding the difference between the two quantities.

From $\productState=(\state,\autoState,0)$, we consider the expected discounted reward conditioned on whether the infinite path following policy $\policy$ satisfies the Büchi condition $\LTL_\autoAccept$ or not.
\begin{align}
G^\policy_0(\productState)&=G^\policy_0(\productState\mid\productState\models\textsf{\upshape G}\textsf{\upshape F}\LTL_\autoAccept)\probP_\policy(\productState\models\textsf{\upshape G}\textsf{\upshape F}\LTL_\autoAccept)\label{satisfy_condition}\\
&+G^\policy_0(\productState\mid\productState\not\models\textsf{\upshape G}\textsf{\upshape F}\LTL_\autoAccept)\probP_\policy(\productState\not\models\textsf{\upshape G}\textsf{\upshape F}\LTL_\autoAccept)\label{unsatisfy_condition}
\end{align}

We first consider the components of Line~\ref{satisfy_condition}. Let the stopping time of first reaching the accepting set $B_\autoAccept$ from $\productState$ be $\tau^{\policy}_{\autoAccept}\defeq\inf\{t>0\mid\productState_t\in B_\autoAccept\}$ and let the first reached state in $B_\autoAccept$ be $(\state,\autoState,n)$. Then, let the hitting time between accepting states in the accepting BSCC $B$ containing $(\state,\autoState,n)\in B^\times_\autoAccept$ be $\tau^{\policy}_{BSCC}$. Furthermore, let ${c^{\policy}_{\autoAccept acc}\defeq\left\lvert t\leq\tau^{\policy}_\autoAccept : \productState_t\in\productAccept\right\rvert}$ and ${c^{\policy}_{\autoAccept rej}\defeq\left\lvert t\leq\tau^{\policy}_\autoAccept : \productState_t\notin\productAccept\right\rvert}$ be the number of accepting and non-accepting states reached before reaching $B_\autoAccept$ from $\productState$ respectively, where ${c^{\policy}_{\autoAccept acc}+c^{\policy}_{\autoAccept rej}=\tau^{\policy}_\autoAccept}$. Note that all the quantities above depend on the starting state $\productState$ but for convenience we omit it in the notations. Then, we have

\begin{align}
&G^\policy_0(\productState\mid\productState\models\textsf{\upshape G}\textsf{\upshape F}\LTL_\autoAccept)\\
&=\expectE_\policy\left[\sum^\infty_{i=0} (\prod_{j=0}^{i-1}\discount(\productState_j))\cdot\productReward_{i}\bigm\vert \productState_0=\productState\models\textsf{\upshape G}\textsf{\upshape F}\LTL_\autoAccept|\right]\\
&\geq \expectE_\policy\left[\gamma^{c^{\policy}_{\autoAccept rej}}\left(\sum^{c^{\policy}_{\autoAccept acc}}_{n=0}(\prod_{j=0}^{n-1}(1-R_j))R_n\right)+\sum^\infty_{i=\tau^{\policy}_{\autoAccept}}(\prod_{j=0}^{i-1}\discount(\productState_j))\cdot\productReward_i\bigm\vert \productState_0=\productState\models\textsf{\upshape G}\textsf{\upshape F}\LTL_\autoAccept) \right]\label{sat_line2}\\
&\geq \expectE_\policy\left[\gamma^{c^{\policy}_{\autoAccept rej}}\left(\sum^{c^{\policy}_{\autoAccept acc}}_{n=0}(\prod_{j=0}^{n-1}(1-R_j))R_n+\sum^\infty_{i=\tau^{\policy}_{\autoAccept}}(\prod_{j=0}^{c^{\policy}_{\autoAccept acc}}(1-R_j)\prod_{k=\tau^{\policy}_{\autoAccept}}^{i-1}\discount(\productState_k))\cdot\productReward_i\right)\bigm\vert \productState_0=\productState\models\textsf{\upshape G}\textsf{\upshape F}\LTL_\autoAccept)\right]\label{sat_line4}\\
&\geq \expectE_\policy\left[\gamma^{c^{\policy}_{\autoAccept rej}}\left(\sum^{c^{\policy}_{\autoAccept acc}}_{n=0}(\prod_{j=0}^{n-1}(1-R_j))R_n+\gamma^{\tau^{\policy}_{BSCC}/U}\sum^\infty_{n=c^{\policy}_{\autoAccept acc}}(\prod_{j=0}^{n-1}(1-R_j))R_n\right)\bigm\vert \productState_0=\productState\models\textsf{\upshape G}\textsf{\upshape F}\LTL_\autoAccept)\right]\label{sat_line7}\\
    &\geq \gamma^{C^{\policy}_{\autoAccept}+N^\policy/U} \left(\sum^\infty_{n=0}(\prod_{j=0}^{n-1}(1-R_j))R_n\right)\label{sat_line8}\\
    &=\gamma^{C^{\policy}_{\autoAccept}+N^\policy/U}\cdot 1 \label{sat_line9}
\end{align}
where $C^{\policy}_{\autoAccept}=\expectE_\policy[c^{\policy}_{\autoAccept rej}\mid\productState_0=\productState\models\textsf{\upshape G}\textsf{\upshape F}\LTL_\autoAccept)]$, $N^\policy=\expectE_\policy[\tau^{\policy}_{BSCC}\mid\productState_0=\productState\models\textsf{\upshape G}\textsf{\upshape F}\LTL_\autoAccept)]$, $R_n$ is the non-zero rewards in Definition~\ref{reward_defn} and we slightly abuse the notation to let $R_n=R_K \;\forall n>K$. The inequality in Line~\ref{sat_line2} holds because, for the first $\tau^{\policy}_{\autoAccept}$ elements in the sum, there are at most $c^{\policy}_{\autoAccept rej}$ zero reward discount $\gamma$ terms in the product and taking them out of the sum leaves only the non-zero reward terms. Line~\ref{sat_line4} holds similarly by first taking $c^{\policy}_{\autoAccept rej}$ discount $\gamma$ terms out of the product for the rest of the sum, leaving the $(1-R_j)$ terms and the rest of the discount factors $\discount(\productState_k)$. Line~\ref{sat_line7} holds by Lemma~\ref{useful_lemma} and the observation that, after reaching the accepting BSCC, each non-zero reward will receive $\gamma^{\tau^{\policy}_{BSCC}}$ additional discount between accepting states. In addition, by summing the infinite sequences we find $U/(1-\gamma^{\tau^{\policy}_{BSCC}}(1-U))\leq \gamma^{\tau^{\policy}_{BSCC}/U}$ for $\gamma\in(0,1)$ and $U\in(0,1)$, upper bounding the additional discount and leaving only the non-zero reward discount factors $1-R_j$. Finally, Line~\ref{sat_line9} holds by induction because the infinite geometric sum $\sum^\infty_{n=0}(1-R_K)^n R_K=1$ and $R_n+(1-R_n)*1=1$ for all $n<K$.

Intuitively, the expected discounted reward for paths with only accepting states is 1, and we lower bound the expected discounted reward for general paths satisfying $\LTL_\autoAccept$ by bounding the amount of discount the reward receives from non-zero reward states.

Next, we consider the components of Line~\ref{unsatisfy_condition}. Let the stopping time of first reaching the rejecting set $B_\varnothing$ from $\productState$ be
$\tau^{\policy}_\varnothing\defeq\inf\{t>0\mid\productState_t\in B_\varnothing\}$ and let ${c^{\policy}_{\varnothing acc}\defeq\left\lvert t\leq\tau^{\policy}_\varnothing: \productState_t\in\productAccept\right\rvert}$ be the number of accepting states reached before reaching $B_\varnothing$ from $\productState$. We similarly omit the dependency on $\productState$ in the notation. Then, we have

\begin{align}
G^\policy_0(\productState\mid\productState\not\models\textsf{\upshape G}\textsf{\upshape F}\LTL_\autoAccept)&=\expectE_\policy\left[ \sum^\infty_{i=0}(\prod_{j=0}^{i-1}\discount(\productState_j))\cdot\productReward_{i}\bigm\vert \productState_0=\productState\not\models\textsf{\upshape G}\textsf{\upshape F}\LTL_\autoAccept)\right]\\
    &\leq \expectE_\policy\left[\sum^{c^{\policy}_{\varnothing acc}}_{n=0}(\prod_{j=0}^{n-1}(1-R_j))\cdot R_n \bigm\vert \productState_0=\productState\not\models\textsf{\upshape G}\textsf{\upshape F}\LTL_\autoAccept)\right]\label{unsat:line2}\\
    &= \sum^{C^\policy_\varnothing}_{n=0}(\prod_{j=0}^{n-1}(1-R_j))\cdot R_n\label{unsat:line3}\\
    &\leq \sum^{C^\policy_\varnothing}_{n=0}(1-U)^{n}\cdot U\label{unsat:line4}\\
    &=1-(1-U)^{C^\policy_\varnothing}
\end{align}
where $C^\policy_\varnothing=\expectE_\policy[c^{\policy}_{\varnothing acc}\mid \productState_0=\productState\not\models\textsf{\upshape G}\textsf{\upshape F}\LTL_\autoAccept]$. In Line~\ref{unsat:line2}, the inequality holds because, by Lemma~\ref{useful_lemma}, once $B_\varnothing$ is reached, no further accepting states with non-zero rewards can be reached and the total reward can be bounded above by omitting the discount factor $\discount$ from non-accepting states and summing the discounted reward only for the $c^{\policy}_{\varnothing acc}$ accepting states reached before $B_\varnothing$. Line~\ref{unsat:line3} holds by taking expectation on $\tau^{\policy}_\varnothing$ and Line~\ref{unsat:line4} holds by induction because $R_n+(1-R_n)*X_1\leq U+(1-U)*X_2$ for all $n\in [0..K]$ if $X_1\leq X_2$, and $U\geq R_n$ is an upper bound for all rewards. Intuitively, we have bounded the amount of non-zero reward received by paths not satisfying $\LTL_\autoAccept$.

Therefore, from Equation~\ref{satisfy_condition} and these technical results, we have a lower bound for the expected discounted reward
\begin{align}
G^\policy_0(\productState)\geq\gamma^{C^{\policy}_{\autoAccept}+N^\policy/U}\cdot\probP_\policy(\productState\models\textsf{\upshape G}\textsf{\upshape F}\LTL_\autoAccept)
\label{lower_bounds}
\end{align}
by assuming $G^\policy_0(\productState\mid\productState\not\models\textsf{\upshape G}\textsf{\upshape F}\LTL_\autoAccept)=0$, and an upper bound of the expected discounted reward:    
\begin{align}
    G^\policy_0(\productState)&\leq\probP_\policy(\productState\models\textsf{\upshape G}\textsf{\upshape F}\LTL_\autoAccept)+(1-(1-U)^{C^\policy_\varnothing})\cdot\probP_\policy(\productState\not\models\textsf{\upshape G}\textsf{\upshape F}\LTL_\autoAccept)\\
    &=\probP_\policy(\productState\models\textsf{\upshape G}\textsf{\upshape F}\LTL_\autoAccept)+(1-(1-U)^{C^\policy_\varnothing})\cdot(1-\probP_\policy(\productState\models\textsf{\upshape G}\textsf{\upshape F}\LTL_\autoAccept))\\
    &=1-(1-U)^{C^\policy_\varnothing}+\probP_\policy(\productState\models\textsf{\upshape G}\textsf{\upshape F}\LTL_\autoAccept)\cdot (1-U)^{C^\policy_\varnothing}
\label{upper_bounds}
\end{align}

Last but not least, since deterministic memoryless policy is considered in a finite MDP, there is a finite set of policies and we let the difference in probability of satisfying $\LTL_\autoAccept$ between the optimal policy and the best sub-optimal policy be
$\delta\defeq\probP_{\policy_\LTL}(\productState_0\models\textsf{\upshape G}\textsf{\upshape F }\LTL_\autoAccept)-\max_{\{\policy\neq\policy_\LTL\}}\probP_\policy(\productState_0\models\textsf{\upshape G}\textsf{\upshape F }\LTL_\autoAccept)$. We can let the reward upper bound $U\in(0,1)$ be small enough and $\discount\in(0,1)$ large enough such that following the bounds of Line~\ref{lower_bounds} and Line~\ref{upper_bounds}, we have that
\begin{equation}
    G^\policy_0(\productState)\leq\max_{\policy\neq\policy_\LTL}\{\probP_\policy(\productState_0\models\textsf{\upshape G}\textsf{\upshape F }\LTL_\autoAccept)\}+\delta=\probP_{\policy_\LTL}(\productState_0\models\textsf{\upshape G}\textsf{\upshape F }\LTL_\autoAccept)\leq G^{\policy_\LTL}_0(\productState) \quad\forall \policy\neq\policy_\LTL
\end{equation}
This means the expected discounted reward for the optimal policy $\policy_\LTL$ must be greater than the expected discounted reward received by any sub-optimal policy. We can therefore conclude that the policy $\policy_\reward$ maximizes the total expected discounted reward $G^{\policy_\reward}_0(\productState)$ is also the optimal policy $\policy_\LTL$ that maximizes the probability of satisfying the Büchi condition $\LTL_\autoAccept$ on the product MDP $\productMDP$, which completes the proof.

For choosing the key parameters $U\in(0,1)$ and $\discount\in(0,1)$ to ensure optimality, if we assume a reasonable gap $\delta\approx0.5$ between the optimal and sub-optimal policies, it generally suffices to let $U=1/C^\policy_\varnothing$ and $\discount=1-1/(C^\policy_\varnothing*N^\policy+C^\policy_\autoAccept)$. For example, with the example probabilistic gate MDP in Figure~\ref{fig:example_mdp}, we have $C^\policy_\varnothing\approx C^\policy_\autoAccept\approx 10$ and $N^\policy=1$, so choosing $U=0.1$ and $\gamma=0.95$ is sufficient. If the hitting times and stopping times are not obtainable, under the common assumption for MDPs that the number of states $\abs{\states}$ in $\MDP$ and the minimum non-zero transition probability $p_{\min}\coloneqq \min_{s,a,s^\prime}\{\transitions(\state,\action,\state^\prime)>0\}$ are known, $C^\policy_\varnothing$ and $C^\policy_\autoAccept$ can be upper bounded by $\abs{\states}/p_{\min}$ and $N^\policy$ can be upper bounded by $\abs{\states}$.
\end{proof}

\subsection{Proof of theorem~\ref{thm:q_learning}}
\label{proof:q_learning}

\begin{proof}
To begin with, recall from Theorem~\ref{thm:equivalence} and Theorem~\ref{thm:optimality} that the policy $\policy_r$ maximizing the expected discounted reward of $\productMDP$ induces the optimal policy for satisfying $\LTL$ in $\MDP$ by removing the $\epsilon$-actions.

Next, note that, despite the reward for Q-learning in Algorithm~\ref{alg:Q_learning} being non-Markovian, following the proof of convergence to optimal Q value for non-Markovian Q-learning~\cite{Majeed2018OnProcesses}, with the $K$ counter in Definition~\ref{def:productMDP} and the reward in Definition~\ref{reward_defn}, we verify that the Q function in Algorithm~\ref{alg:Q_learning} converges to the optimal Q value for all $(\state,\autoState)$ and action $\productAction$, which is the expected discounted reward by taking action $\productAction$ from $(\state,\autoState,0)$.

With this, recall that the environment states and automaton states transitions do not depend on the counter value by Equation~\ref{eq_lemma} and the non-zero reward is the same for accepting states with the same counter value by Definition~\ref{reward_defn}. We argue that, for the optimal discounted reward policy $\policy_r$, the new policy $\policy^\prime((\state,\autoState,n))=\policy_r((\state,\autoState,0))\;\forall n\in[0..K]$, where the policy for each counter values is set to be the same, remains optimal. This is because, if $\policy^\prime$ is sub-optimal, meaning the satisfiability of $\LTL_\autoAccept$ is lower than that of the optimal policy, it must be that $\policy^\prime$ makes a wrong decision at some counter value such that the reachability probability to some accepting BSCC $B$ is sub-optimal. This means $\policy_r$ must also be sub-optimal because when counter value equals 0, by taking the same actions as $\policy^\prime$, the reachability probability to the BSCC $B$ is sub-optimal, yielding a contradiction. Lastly, since $\policy^\prime$ is independent from the counter values, the policy derived from the collapsed optimal Q function learned by Algorithm~\ref{alg:Q_learning} is the same as the optimal $\policy^\prime$ by ignoring the $K$ counter, meaning that Algorithm~\ref{alg:Q_learning} converges to the optimal policy, completing the proof.

\end{proof}

\section{Additional Experimental Details And Setup}
\label{appendix:exp_setup}

Experiments are carried out on a Linux server (Ubuntu 18.04.2) with two Intel Xeon Gold 6252 CPUs and six NVIDIA GeForce RTX 2080 Ti GPUs. All our algorithms are implemented in Python, and the full source code can be found in our supplementary material. We select three stochastic environments (probabilistic gate, frozen lake and office world) described below with several difficult LTL tasks. For all environments and tasks we set the learning rate $\alpha=0.1$ and the exploration rate $\epsilon=0.1$, with the upper bound on rewards $U=0.1$ and discount factor $\gamma=0.99$ to ensure optimality. For the method proposed by Hahn \textit{et al.}~\shortcite{Hahn2019Omega-regularLearning}, we adopt their implementation from the Mungojerrie tool\footnote{https://plv.colorado.edu/wwwmungojerrie/docs/v1\_0/index.html}. We set the default parameter values to be $\alpha=0.1$ and $\epsilon=0.1$, with $\zeta=0.995$ for the probabilistic gate task and $\zeta=0.99$ for the frozen lake and office world tasks. For the method by Hasanbeig \textit{et al.}~\shortcite{Hasanbeig2020DeepLogics}, we use the implementation in their GitHub repository\footnote{https://github.com/grockious/lcrl} with $\gamma=0.99$, $\alpha=0.1$ and $\epsilon=0.1$.
For the method by Bozkurt \textit{et al.}~\shortcite{Bozkurt2019ControlLearning}, we also adopt their GitHub repository\footnote{https://github.com/alperkamil/csrl} with $\gamma=0.99$, $\gamma_B=0.9$, $\alpha=0.1$ and $\epsilon=0.1$.

For the environments and tasks used for experimentation, we specifically choose commonly used environments with stochastic transitions that are more realistic and challenging, especially for infinite-horizon LTL tasks. The LTL tasks we use are standard task specifications for robotics and autonomous systems. The tasks for the frozen lake and the office world environments are more challenging because they involve letting the agent choose between two subtasks, which better showcase the benefits of the $K$ counter and counterfactual imagining.

\subsection{Probabilistic Gate}
\label{appendix:probablistic}

\begin{figure*}[tb]
\centering
\includegraphics[width=0.5\textwidth]{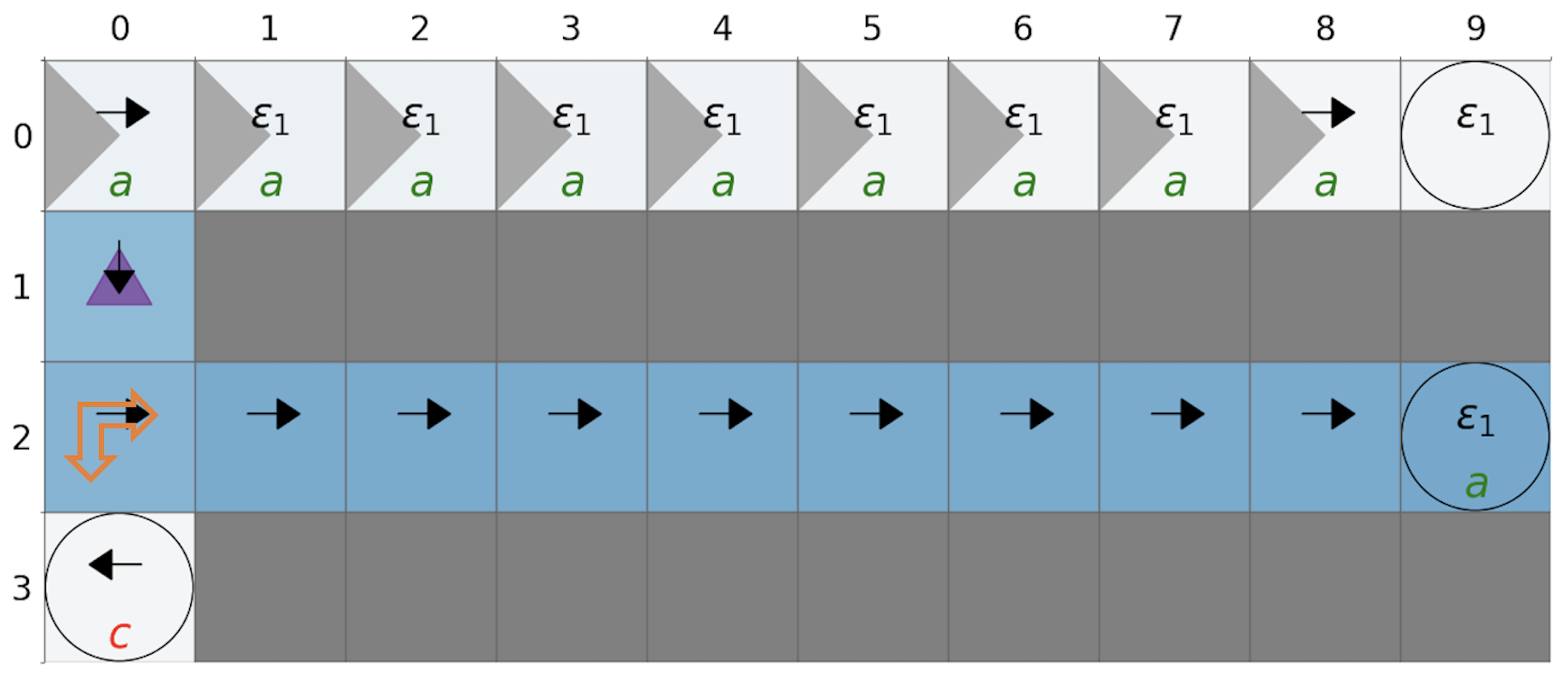}
\caption{The optimal policy for the example probabilistic gate MDP task\ref{fig:example_task}}
\label{fig:hard1_policy}
\end{figure*}

The probabilistic gate MDP is described in Example~\ref{example_mdp}, where the task is to visit states labeled \enquote{a} infinitely often without visiting states labeled \enquote{c}. We set the episode length to be 100 (which is the maximum length of a path in the MDP) and the number of episodes to be 40000. In Figure~\ref{fig:hard1_policy}, the optimal policy is demonstrated by the black arrows, where the agent goes down from the start and before going all the way to the right to (2,9), which is an accepting sink state.

\subsection{Frozen Lake}
\label{appendix:frozenlake}

\begin{figure*}[tb]
\centering
\begin{subfigure}[t]{0.45\textwidth}
\centering
\includegraphics[width=0.8\textwidth]{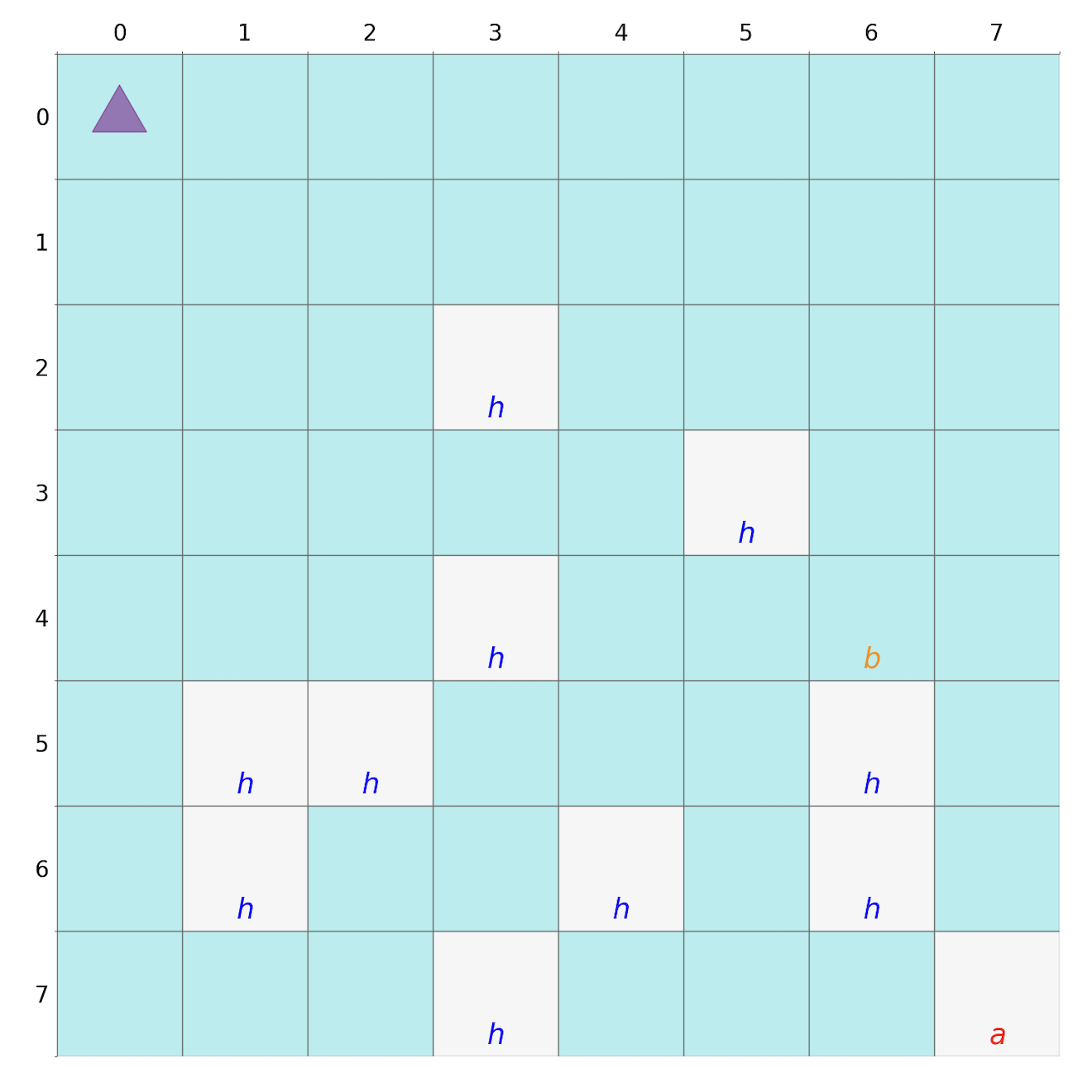}
\caption{The MDP environment for the frozen lake task. Blue represents ice, \enquote{h} are holes, \enquote{a} and \enquote{b} are lake camps, and the purple triangle is the start.}
\label{fig:frozen_env}
\end{subfigure}
\hspace{0.5cm}
\begin{subfigure}[t]{0.5\textwidth}
\centering
\includegraphics[width=\textwidth]{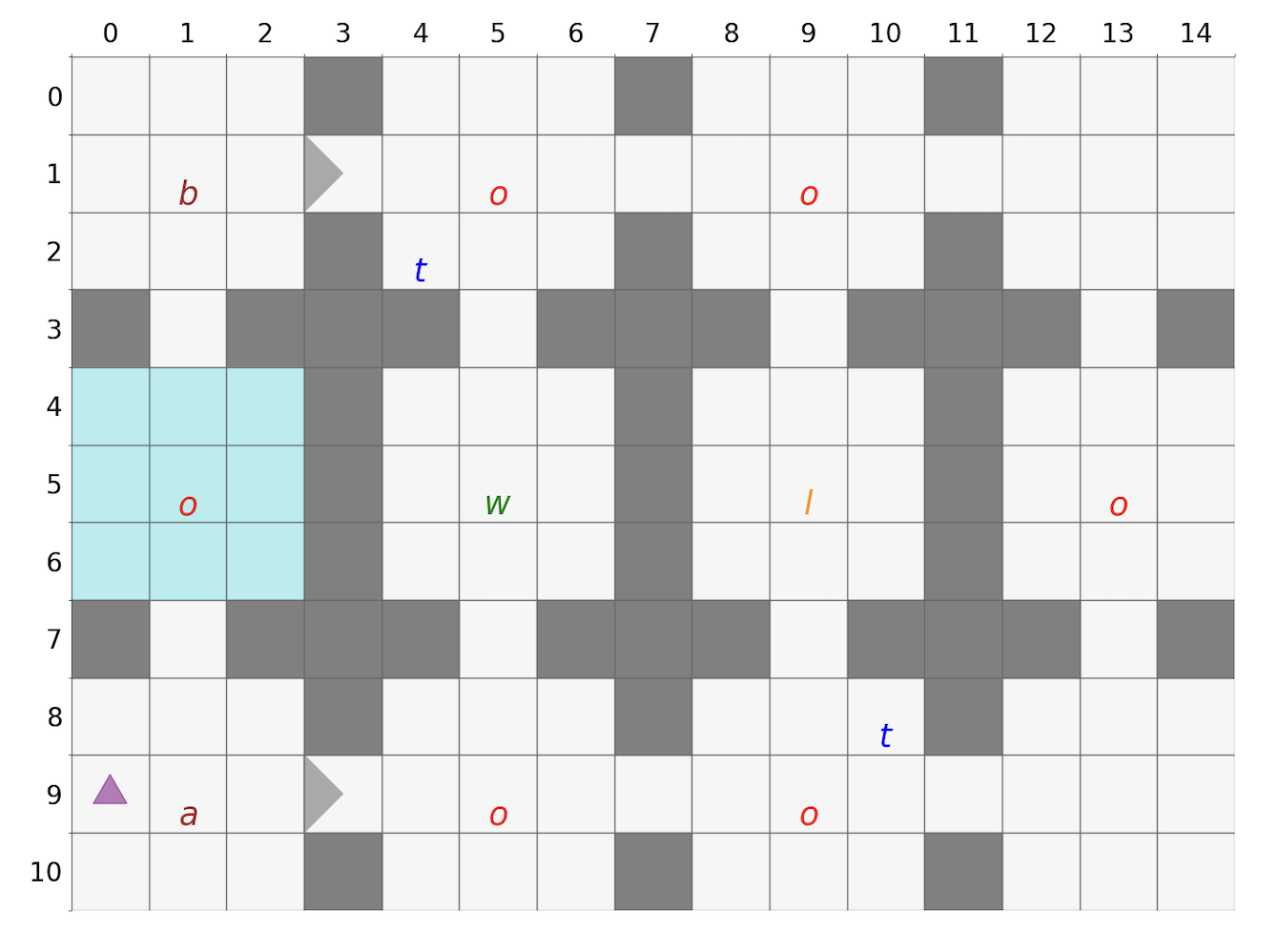}
\caption{The MDP environment for the office world task. Blue represents ice, \enquote{o} are obstacles, \enquote{w}, \enquote{l} and \enquote{t} represent workplace, letter and tea respectively, and the purple triangle is the start.}
\label{fig:office_env}
\end{subfigure}
\caption{MDP environments used in the experiments.}
\end{figure*}

The frozen lake~\cite{Brockman2016OpenAIGym} environment is shown in Figure~\ref{fig:frozen_env}. The blue states represent the frozen lake, where the agent has 1/3 probability of moving in the intended direction and 1/3 each of going sideways (left or right). The white states with label \enquote{h} are holes and states with label \enquote{a} and \enquote{b} are lake camps. The task is \enquote{(\textsf{\upshape G}\textsf{\upshape F} a $\mid$ \textsf{\upshape G}\textsf{\upshape F} b) \& \textsf{\upshape G} !h}, meaning to always reach lake camp \enquote{a} or lake camp \enquote{b} while never falling into holes \enquote{h}. For this task, we set the episode length to be 200 and the number of episodes to be 6000.

\subsection{Office World}
\label{appendix:officeworld}

The office world~\cite{Icarte2022RewardLearning} environment is demonstrated in Figure~\ref{fig:office_env}. We include a patch of ice labeled in blue as defined in the frozen lake environment and two one-directional gates at $(1,3)$ and $(9,3)$, where the agent can only cross to the right as shown by the direction of the gray triangle. The gray blocks are walls, the states labeled \enquote{o}, \enquote{l}, \enquote{t} and \enquote{w} are obstacles, letters, tea and the workplace respectively. The task for this environment is also demanding: {\enquote{(\textsf{\upshape G}\textsf{\upshape F} a \& \textsf{\upshape G}\textsf{\upshape F} b) $\mid$ (\textsf{\upshape F} l \& \textsf{\upshape X} (\textsf{\upshape G}\textsf{\upshape F} t \& \textsf{\upshape G}\textsf{\upshape F} w)) \& \textsf{\upshape G} !o}}. This means to either patrol in the corridor between \enquote{a} and \enquote{b}, or go to write letter at \enquote{l} and then patrol between getting tea \enquote{t} and workplace \enquote{w}, whilst never hitting obstacles \enquote{o}. For this task, we set the episode length to be 1000 and the number of episodes to be 6000.

\subsection{PRISM}
We use PRISM version 4.7~\cite{Kwiatkowska2011PRISMSystems} to evaluate the learned policy against the LTL tasks. The PRISM tool and the installation documentations can be obtained from the their official website\footnote{https://www.prismmodelchecker.org/download.php}. Each time we would like to evaluate the policy on the environment MDP, our Python code autonomously constructs the induced Markov chain from the MDP and the policy as a discrete-time Markov chain in the PRISM language, and evaluates this PRISM model against the LTL task specification through a call from Python. The max iteration parameter for PRISM is set to 100000, and we evaluate the current policy every 10000 training steps for plotting the training graph for all our experiments.

\subsubsection{Example PRISM model}
We provide an example PRISM model of the induced Markov chain from the probabilistic gate MDP and its optimal policy.
\begin{lstlisting}[caption={The PRISM model of the induced Markov chain from the probabilistic gate MDP and its optimal policy.},label={lst:PRISM},language=c++]
dtmc
module ProductMDP
    m : [0..40] init 10;
    a : [0..3] init 0;
    [ep] (m=0)&(a=0) -> (m'=0)&(a'=1);
    [ac] (m=0)&(a=1) -> 1 : (m'=1)&(a'=1);
    [ac] (m=0)&(a=2) -> 1 : (m'=1)&(a'=2);
    [ep] (m=1)&(a=0) -> (m'=1)&(a'=1);
    [ac] (m=1)&(a=1) -> 1 : (m'=2)&(a'=1);
    [ac] (m=1)&(a=2) -> 1 : (m'=2)&(a'=2);
    [ep] (m=2)&(a=0) -> (m'=2)&(a'=1);
    [ac] (m=2)&(a=1) -> 1 : (m'=3)&(a'=1);
    [ac] (m=2)&(a=2) -> 1 : (m'=3)&(a'=2);
    [ep] (m=3)&(a=0) -> (m'=3)&(a'=1);
    [ac] (m=3)&(a=1) -> 1 : (m'=4)&(a'=1);
    [ac] (m=3)&(a=2) -> 1 : (m'=4)&(a'=2);
    [ep] (m=4)&(a=0) -> (m'=4)&(a'=1);
    [ac] (m=4)&(a=1) -> 1 : (m'=5)&(a'=1);
    [ac] (m=4)&(a=2) -> 1 : (m'=5)&(a'=2);
    [ep] (m=5)&(a=0) -> (m'=5)&(a'=1);
    [ac] (m=5)&(a=1) -> 1 : (m'=6)&(a'=1);
    [ac] (m=5)&(a=2) -> 1 : (m'=6)&(a'=2);
    [ep] (m=6)&(a=0) -> (m'=6)&(a'=1);
    [ac] (m=6)&(a=1) -> 1 : (m'=7)&(a'=1);
    [ac] (m=6)&(a=2) -> 1 : (m'=7)&(a'=2);
    [ep] (m=7)&(a=0) -> (m'=7)&(a'=1);
    [ac] (m=7)&(a=1) -> 1 : (m'=8)&(a'=1);
    [ac] (m=7)&(a=2) -> 1 : (m'=8)&(a'=2);
    [ep] (m=8)&(a=0) -> (m'=8)&(a'=1);
    [ac] (m=8)&(a=1) -> 1 : (m'=9)&(a'=1);
    [ac] (m=8)&(a=2) -> 1 : (m'=9)&(a'=2);
    [ep] (m=9)&(a=0) -> (m'=9)&(a'=1);
    [ac] (m=9)&(a=1) -> 1.0 : (m'=9)&(a'=2);
    [ac] (m=9)&(a=2) -> 1.0 : (m'=9)&(a'=2);
    [ac] (m=10)&(a=0) -> 1 : (m'=20)&(a'=0);
    [ac] (m=10)&(a=1) -> 1 : (m'=20)&(a'=2);
    [ac] (m=10)&(a=2) -> 1 : (m'=20)&(a'=2);
    [ac] (m=11)&(a=0) -> 1.0 : (m'=11)&(a'=0);
    [ac] (m=11)&(a=1) -> 1.0 : (m'=11)&(a'=2);
    [ac] (m=11)&(a=2) -> 1.0 : (m'=11)&(a'=2);
    [ac] (m=12)&(a=0) -> 1.0 : (m'=12)&(a'=0);
    [ac] (m=12)&(a=1) -> 1.0 : (m'=12)&(a'=2);
    [ac] (m=12)&(a=2) -> 1.0 : (m'=12)&(a'=2);
    [ac] (m=13)&(a=0) -> 1.0 : (m'=13)&(a'=0);
    [ac] (m=13)&(a=1) -> 1.0 : (m'=13)&(a'=2);
    [ac] (m=13)&(a=2) -> 1.0 : (m'=13)&(a'=2);
    [ac] (m=14)&(a=0) -> 1.0 : (m'=14)&(a'=0);
    [ac] (m=14)&(a=1) -> 1.0 : (m'=14)&(a'=2);
    [ac] (m=14)&(a=2) -> 1.0 : (m'=14)&(a'=2);
    [ac] (m=15)&(a=0) -> 1.0 : (m'=15)&(a'=0);
    [ac] (m=15)&(a=1) -> 1.0 : (m'=15)&(a'=2);
    [ac] (m=15)&(a=2) -> 1.0 : (m'=15)&(a'=2);
    [ac] (m=16)&(a=0) -> 1.0 : (m'=16)&(a'=0);
    [ac] (m=16)&(a=1) -> 1.0 : (m'=16)&(a'=2);
    [ac] (m=16)&(a=2) -> 1.0 : (m'=16)&(a'=2);
    [ac] (m=17)&(a=0) -> 1.0 : (m'=17)&(a'=0);
    [ac] (m=17)&(a=1) -> 1.0 : (m'=17)&(a'=2);
    [ac] (m=17)&(a=2) -> 1.0 : (m'=17)&(a'=2);
    [ac] (m=18)&(a=0) -> 1.0 : (m'=18)&(a'=0);
    [ac] (m=18)&(a=1) -> 1.0 : (m'=18)&(a'=2);
    [ac] (m=18)&(a=2) -> 1.0 : (m'=18)&(a'=2);
    [ac] (m=19)&(a=0) -> 1.0 : (m'=19)&(a'=0);
    [ac] (m=19)&(a=1) -> 1.0 : (m'=19)&(a'=2);
    [ac] (m=19)&(a=2) -> 1.0 : (m'=19)&(a'=2);
    [ac] (m=20)&(a=0) -> 0.8 : (m'=21)&(a'=0) + 0.2 : (m'=30)&(a'=0);
    [ac] (m=20)&(a=1) -> 0.8 : (m'=21)&(a'=2) + 0.2 : (m'=30)&(a'=2);
    [ac] (m=20)&(a=2) -> 0.8 : (m'=21)&(a'=2) + 0.2 : (m'=30)&(a'=2);
    [ac] (m=21)&(a=0) -> 1 : (m'=22)&(a'=0);
    [ac] (m=21)&(a=1) -> 1.0 : (m'=21)&(a'=2);
    [ac] (m=21)&(a=2) -> 1.0 : (m'=21)&(a'=2);
    [ac] (m=22)&(a=0) -> 1 : (m'=23)&(a'=0);
    [ac] (m=22)&(a=1) -> 1.0 : (m'=22)&(a'=2);
    [ac] (m=22)&(a=2) -> 1.0 : (m'=22)&(a'=2);
    [ac] (m=23)&(a=0) -> 1 : (m'=24)&(a'=0);
    [ac] (m=23)&(a=1) -> 1.0 : (m'=23)&(a'=2);
    [ac] (m=23)&(a=2) -> 1.0 : (m'=23)&(a'=2);
    [ac] (m=24)&(a=0) -> 1 : (m'=25)&(a'=0);
    [ac] (m=24)&(a=1) -> 1.0 : (m'=24)&(a'=2);
    [ac] (m=24)&(a=2) -> 1.0 : (m'=24)&(a'=2);
    [ac] (m=25)&(a=0) -> 1 : (m'=26)&(a'=0);
    [ac] (m=25)&(a=1) -> 1.0 : (m'=25)&(a'=2);
    [ac] (m=25)&(a=2) -> 1.0 : (m'=25)&(a'=2);
    [ac] (m=26)&(a=0) -> 1 : (m'=27)&(a'=0);
    [ac] (m=26)&(a=1) -> 1.0 : (m'=26)&(a'=2);
    [ac] (m=26)&(a=2) -> 1.0 : (m'=26)&(a'=2);
    [ac] (m=27)&(a=0) -> 1 : (m'=28)&(a'=0);
    [ac] (m=27)&(a=1) -> 1.0 : (m'=27)&(a'=2);
    [ac] (m=27)&(a=2) -> 1.0 : (m'=27)&(a'=2);
    [ac] (m=28)&(a=0) -> 1 : (m'=29)&(a'=0);
    [ac] (m=28)&(a=1) -> 1 : (m'=27)&(a'=2);
    [ac] (m=28)&(a=2) -> 1 : (m'=27)&(a'=2);
    [ep] (m=29)&(a=0) -> (m'=29)&(a'=1);
    [ac] (m=29)&(a=1) -> 1.0 : (m'=29)&(a'=1);
    [ac] (m=29)&(a=2) -> 1.0 : (m'=29)&(a'=2);
    [ep] (m=30)&(a=0) -> (m'=30)&(a'=1);
    [ac] (m=30)&(a=1) -> 1.0 : (m'=30)&(a'=2);
    [ac] (m=30)&(a=2) -> 1.0 : (m'=30)&(a'=2);
    [ac] (m=31)&(a=0) -> 1.0 : (m'=31)&(a'=0);
    [ac] (m=31)&(a=1) -> 1.0 : (m'=31)&(a'=2);
    [ac] (m=31)&(a=2) -> 1.0 : (m'=31)&(a'=2);
    [ac] (m=32)&(a=0) -> 1.0 : (m'=32)&(a'=0);
    [ac] (m=32)&(a=1) -> 1.0 : (m'=32)&(a'=2);
    [ac] (m=32)&(a=2) -> 1.0 : (m'=32)&(a'=2);
    [ac] (m=33)&(a=0) -> 1.0 : (m'=33)&(a'=0);
    [ac] (m=33)&(a=1) -> 1.0 : (m'=33)&(a'=2);
    [ac] (m=33)&(a=2) -> 1.0 : (m'=33)&(a'=2);
    [ac] (m=34)&(a=0) -> 1.0 : (m'=34)&(a'=0);
    [ac] (m=34)&(a=1) -> 1.0 : (m'=34)&(a'=2);
    [ac] (m=34)&(a=2) -> 1.0 : (m'=34)&(a'=2);
    [ac] (m=35)&(a=0) -> 1.0 : (m'=35)&(a'=0);
    [ac] (m=35)&(a=1) -> 1.0 : (m'=35)&(a'=2);
    [ac] (m=35)&(a=2) -> 1.0 : (m'=35)&(a'=2);
    [ac] (m=36)&(a=0) -> 1.0 : (m'=36)&(a'=0);
    [ac] (m=36)&(a=1) -> 1.0 : (m'=36)&(a'=2);
    [ac] (m=36)&(a=2) -> 1.0 : (m'=36)&(a'=2);
    [ac] (m=37)&(a=0) -> 1.0 : (m'=37)&(a'=0);
    [ac] (m=37)&(a=1) -> 1.0 : (m'=37)&(a'=2);
    [ac] (m=37)&(a=2) -> 1.0 : (m'=37)&(a'=2);
    [ac] (m=38)&(a=0) -> 1.0 : (m'=38)&(a'=0);
    [ac] (m=38)&(a=1) -> 1.0 : (m'=38)&(a'=2);
    [ac] (m=38)&(a=2) -> 1.0 : (m'=38)&(a'=2);
    [ac] (m=39)&(a=0) -> 1.0 : (m'=39)&(a'=0);
    [ac] (m=39)&(a=1) -> 1.0 : (m'=39)&(a'=2);
    [ac] (m=39)&(a=2) -> 1.0 : (m'=39)&(a'=2);
endmodule

label "a" = (m=0) | (m=1) | (m=2) | (m=3) | (m=4) | (m=5) | (m=6) | (m=7) | (m=8) | (m=29);
label "c" = (m=30);
\end{lstlisting}

\subsection{Rabinizer}
We use Rabinizer 4~\cite{Kretinsky2018RabinizerAutomaton} to transform LTL formulae into LDBAs. The tool can be downloaded on their website~\footnote{https://www7.in.tum.de/~kretinsk/rabinizer4.html}, and we use the 'ltl2ldba' script with '-e' and '-d' options to construct the LDBA with $\epsilon$-transitions from the input LTL formulae.